\documentclass{article}
\usepackage{ijcai22}

\usepackage{times}
\usepackage{soul}
\usepackage{url}
\usepackage[hidelinks]{hyperref}
\usepackage[utf8]{inputenc}
\usepackage[small]{caption}
\usepackage{graphicx}
\usepackage{amsmath}
\usepackage{amsthm}
\usepackage{booktabs}
\usepackage{algorithm}
\usepackage{algorithmic}
\usepackage{amsmath}
\usepackage{amssymb}
\usepackage{enumitem}
\usepackage{graphicx}
\usepackage{subcaption}
\usepackage{multirow}
\usepackage{color}
\newtheorem{theorem}{Theorem}

\newtheorem{lemma}{Lemma}
\newtheorem{proposition}{Proposition}
\newtheorem*{atheorem}{Theorem}
\newtheorem*{aproposition}{Proposition}

\title{Penalized  Proximal Policy Optimization for Safe Reinforcement Learning}

\author{
Linrui Zhang$^{1}$\thanks{This work was done during Linrui Zhang's internship at JD Explore Academy, Beijing, China.}
\and
Li Shen$^2$\thanks{Corresponding Author.}\and
Long Yang$^3$\and
Shixiang Chen$^2$ \\
Xueqian Wang$^{1}$\and
Bo Yuan$^1$\and
Dacheng Tao$^2$\
\affiliations
$^1$Tsinghua Shenzhen International Graduate School, Tsinghua University\\
$^2$JD Explore Academy\\
$^3$ School of Artificial Intelligence, Peking University
\emails
{zlr20@mails.tsinghua.edu.cn, \{mathshenli,chenshxiang,dacheng.tao\}@gmail.com, yanglong001@pku.edu.cn,
 boyuan@ieee.org,
wang.xq@sz.tsinghua.edu.cn}
}
\begin{document}

\maketitle

\begin{abstract}
 Safe reinforcement learning aims to learn the optimal policy while satisfying safety constraints, which is essential in real-world applications. However, current algorithms still struggle for efficient policy updates with hard constraint satisfaction. In this paper, we propose Penalized Proximal Policy Optimization (P3O), which solves the cumbersome constrained policy iteration via a single minimization of an equivalent unconstrained problem. Specifically, P3O utilizes a simple-yet-effective penalty function to eliminate cost constraints and removes the trust-region constraint by the clipped surrogate objective. We theoretically prove the exactness of the proposed method with a finite penalty factor and provide a worst-case analysis for approximate error when evaluated on sample trajectories. Moreover, we extend P3O to more challenging multi-constraint and multi-agent scenarios which are less studied in previous work. Extensive experiments show that P3O outperforms state-of-the-art algorithms with respect to both reward improvement and constraint satisfaction on a set of constrained locomotive tasks.
\end{abstract}
\section{Introduction}

Reinforcement Learning (RL) has achieved significant successes in playing video games~\cite{vinyals2019grandmaster}, robotic manipulation~\cite{levine2016end}, mastering Go~\cite{silver2017mastering}, etc. 
However, standard RL merely maximizes cumulative rewards, which may lead to undesirable behaviors in real-world applications especially for safety-critical tasks.

%, for example, service robots would not avoid harm to humans deliberately when completing their tasks.

It is intractable to learn reasonable policies by directly penalizing unsafe interactions onto the reward function, since varied intensities of the punishment result in different Markov decision processes. There is still a lack of theories revealing the explicit relationship between policy improvement and safety satisfaction via reward shaping.

Constrained Markov decision process (CMDP)~\cite{altman1999constrained} is a more practical and popular formulation that requires the agent to perform actions under given constraints. A common approach to solving such constrained sequential optimization is approximating the non-convex constraint function to a quadratic optimization problem via Taylor's formulation~\cite{achiam2017constrained,yang2020projection}.
However, those algorithms still have the following drawbacks: (1) The convex approximation to a non-convex policy optimization results in non-negligible approximate errors and only learns near constraint-satisfying policies;
 (2) The closed-form solution to the primal problem involves the inversion of Hessian matrix, which is computationally expensive in large CMDPs with deep neural networks; (3) If there is more than one constraint, it is cumbersome to obtain the analytical solution without an extra inner-loop optimization, which limits their applications in multi-constraint and multi-agent scenarios; (4) If the primal problem is infeasible under certain initial policies, above algorithms require additional interactions for feasible recovery, which reduces their sample-efficiency.

To address above issues, we propose \textbf{P}enalized \textbf{P}roximal \textbf{P}olicy \textbf{O}ptimization (P3O) algorithm that contains two key techniques. Firstly, we employ exact penalty functions to derive an  equivalent unconstrained optimization problem that is naturally compatible with multiple constraints and arbitrary initial policies. Secondly, we extend the clipped surrogate objective in Proximal Policy Optimization (PPO)~\cite{schulman2017proximal} to CMDPs, which eliminates the trust-region constraint on both reward term and cost term. As a consequence, our method removes the need for quadratic approximation and Hessian matrix inversion. By minimizing the unconstrained objective instead, the solution to the primal constrained problem can be obtained.

Conclusively, the proposed P3O algorithm has three main strengths:
(1) \textit{accuracy}: P3O algorithm is accurate and sample-efficient with first-order optimization over a theoretically equivalent unconstrained objective instead of solving an approximate quadratic optimization;
(2) \textit{feasibility}: P3O algorithm could admit arbitrary initialization with a consistent objective function and doesn't need additional recovery methods for infeasible policies;
(3) \textit{scalability}: P3O is naturally scalable to multiple constraints or agents without increasing complexity due to the exact penalty reformulation.

In the end, we summarize our contributions as four-fold:
\begin{itemize}
\item  We propose the Penalized Proximal Policy Optimization (P3O) algorithm for safe RL, which employs first-order optimization over an equivalent  unconstrained objective.
\item  We theoretically prove the exactness of the penalized method when the penalty factor is sufficiently large (doesn't have to go towards positive infinity) and propose an adaptive penalty factor tuning algorithm.
\item We extend clipped surrogate objectives and advantage normalization tricks to CMDPs, which are easy to implement and enable a fixed penalty factor for general good results across different tasks.
\item We conduct extensive experiments to show that P3O outperforms several state-of-the-art algorithms with respect to both reward improvement and constraint satisfaction and demonstrate its efficacy in more difficult multi-constraint and multi-agent scenarios which are less studied in previous safe RL algorithms.
\end{itemize}

\section{Related Work}

Safety in reinforcement learning is a challenging topic formally raised by \citeauthor{garcia2015comprehensive} \shortcite{garcia2015comprehensive}. Readers can  refer to the survey~\cite{liu2021policy} for recent advances in safe RL. In this section, we only summarize the most related studies to our algorithm. 

\paragraph{Primal-Dual solution.}

Considering the success of Lagrangian relaxation in solving constrained optimization problems, Primal-Dual Optimization (PDO)~\cite{chow2017risk} and its variants~\cite{tessler2018reward,stooke2020responsive} leverage the Lagrangian Duality in constrained reinforcement learning. Although plenty of work has provided rigorous analysis on duality gap~\cite{paternain2019constrained}, non-asymptotic convergence rate~\cite{ding2020natural} and regret bound~\cite{ding2021provably}, Primal-dual methods are still hard to be implemented and applied in practical use due to their sensitivity to the initialization as well as the learning rate of Lagrangian multipliers.

\paragraph{Primal solution.}
Constrained Policy Optimization (CPO)~\cite{achiam2017constrained} directly searches the feasible policy in the trust region and guarantees a monotonic performance improvement while ensuring constraint satisfaction by solving a quadratic optimization problem with appropriate approximations. Projection-based CPO~\cite{yang2020projection} updates the policy in two stages by firstly performing regular TRPO~\cite{schulman2015trust} and secondly projecting the policy back into the constraint set. Those methods based on local policy search mainly suffer from approximate errors, the inversion of high-dimensional Hessian matrices, and the poor scalability to multiple constraints. To address drawbacks of the quadratic approximation, FOCOPS~\cite{zhang2020first} solves the optimum of constrained policy optimization within the non-parametric space and then derives the first-order gradients of the $\ell_2$ loss function within the parameter space. Nevertheless, FOCOPS has more auxiliary variables to learn than our first-order optimization objective, and the analytical solutions of FOCOPS are not straightforward with the increasing number of constraints.

The most similar work to our proposed algorithm is Interior-point Policy Optimization (IPO)~\cite{liu2020ipo} which uses log-barrier functions as penalty terms to restrict policies into the feasible set. However, the interior-point method requires a feasible policy upon initialization which is not necessarily fulfilled and needs a further recovery. Moreover, the log-barrier function possibly leads to numerical issues when the penalty term goes towards infinity, or the solution is not exactly the same as the primal problem. By contrast, we employ the exact penalty function to derive an equivalent unconstrained objective and restrict policy updates in the trust region by clipping the important sampling ratio on both reward and cost terms for the approximate error reduction, which is different from the prior work.

\section{Preliminaries}
A Markov decision process (MDP)~\cite{sutton2018reinforcement} is a tuple $(S,A,R,P,\mu)$ ,where $S$ is the state space, $A$ is the action space, $R:  S \times  A \times  S  \mapsto \mathbb{R}$ is the reward function, $P: S \times  A \times  S \mapsto [0,1]$ is the transition  probability  function to describe the dynamics of the environment, and $\mu : S  \mapsto [0,1] $ is the initial state distribution. A stationary policy $\pi : S \mapsto \cal{P}(A)$ maps the given states to probability distributions over action space and  $\Pi$ denotes the set of all stationary policies $\pi$.
The optimal policy $\pi^*$ maximizes the expected discounted return
$J_R(\pi) = \mathop{\mathbb{E}}_{\tau\sim \pi}\big [ \sum^\infty_{t=0}\gamma^t R(s_t,a_t,s_{t+1})\big ]$ where $\tau=\{(s_t,a_t)\}_{t\ge0}$ is a sample trajectory and $\tau\sim\pi$ indicates the distribution over trajectories depending on $s_0 \sim \mu, a_t \sim \pi(\cdot | s_t), s_{t+1} \sim P(\cdot | s_t,a_t)$.
The value function is defined as $V_R^\pi (s) = \mathop{\mathbb{E}}_{\tau\sim \pi}\big [ \sum^\infty_{t=0}\gamma^t R(s_t,a_t,s_{t+1}) | s_0 = s\big ]$, the action-value function is defined as $Q_R^\pi (s,a) = \mathop{\mathbb{E}}_{\tau\sim \pi}\big [ \sum^\infty_{t=0}\gamma^t R(s_t,a_t,s_{t+1}) | s_0 = s,a_0 =a \big]$ and the advantage function is defined as $A_R^\pi(s,a) = Q_R^\pi(s,a) - V_R^\pi(s)$.

A constrained Markov decision process (CMDP)~\cite{altman1999constrained} extends MDP with a constraint set $C$ consisting of cost functions $C_i : S \times A\times S\mapsto \mathbb{R},\ i=1,2,...,m$. We define the expected discounted cost-return $J_{C_i}(\pi) = \mathop{\mathbb{E}}_{\tau\sim \pi}\big [ \Sigma^\infty_{t=0}\gamma^t C_i(s_t,a_t,s_{t+1})\big ]$ and $V_{C_i}^\pi, Q_{C_i}^\pi,A_{C_i}^\pi$ similarly in MDPs. The set of feasible stationary policies for a CMDP is denoted as $\Pi_C = \{ \pi \in \Pi \ | J_{C_i}(\pi) \leq d_i,\ \forall i\}$, where $d_i$ is the upper limit of the corresponding safety or cost constraint. The goal of safe RL is to find an optimal policy over the hard safe constraint, i.e.,  $\pi^* = \mathop{\arg\max}_{\pi \in \Pi_C} J_R(\pi)$.
\section{Methodology}
Iterative search for optimal policy is commonly used in literature~\cite{schulman2015trust,achiam2017constrained,zhang2020first}.
We formulate the unbiased constrained policy optimization problem over $s\sim d^{\pi}$ and $a\sim \pi$ as:
\begin{equation}
%\resizebox{.91\linewidth}{!}{$
\begin{aligned}
&\pi_{k+1}  =  \mathop{\arg\max}_{\pi}\mathop{\mathbb{E}}_{\substack{s\sim d^{\pi} \\a\sim \pi}}[ A^{\pi_k}_R (s,a)]\\
& \mathrm{s.t.}  \quad J_{C_i}(\pi_{k})+ \frac{1}{1-\gamma}\mathop{\mathbb{E}}_{\substack{s\sim d^{\pi} \\a\sim \pi}} \big  [A_{C_i}^{\pi_ k} (s,a) \big ] \leq d_i,\ \forall i.\\
\end{aligned}
%$}
\label{cppo1}
\end{equation}
where $d^\pi(s) = (1-\gamma) \sum^\infty_{t=0} \gamma^t P(s_t=s | \pi)$ denote the discounted future state distribution.

\begin{proposition}
The new policy $\pi_{k+1}$  obtained from the current policy $\pi_k$ via problem \eqref{cppo1} yields a monotonic return improvement and hard constraint satisfaction.
\end{proposition}
\begin{proof}
See the supplemental material.
\end{proof}

In this paper, we consider the parametric policy  $\pi(\theta)$,  i.e., using the neural network. Let $r(\theta) = \frac{\pi(\theta)}{\pi(\theta_k)}$ be the importance sampling ratio, then we rewrite the optimization problem (\ref{cppo1}) as follows:
\begin{equation}
\resizebox{.91\linewidth}{!}{$
\begin{aligned}
 &\theta_{k+1} =  \mathop{\arg\min}_{\theta}\mathop{\mathbb{E}}_{\substack{s\sim d^{\pi} \\a\sim \pi_{k}}}[-r(\theta)A^{\pi_k}_R (s,a)]\\
& \mathrm{s.t.} \ \mathop{\mathbb{E}}_{\substack{s\sim d^{\pi} \\a\sim \pi_{k}}} \big  [r(\theta)A_{C_i}^{\pi_ k} (s,a) \big ] + (1-\gamma)(J_{C_i}(\pi_{k})-d_i) \leq 0, \ \forall i.\\
\end{aligned}
$}
\label{cppo2}
\end{equation}

Different from previous methods \cite{achiam2017constrained,yang2020projection} that approximate non-convex problem (\ref{cppo2}) to convex optimization via Taylor's formulation on the trust region, we penalize constraints with ReLU operators to the objective function which yields an unconstrained problem:
\begin{equation}
\resizebox{.91\linewidth}{!}{$
\begin{aligned}
\theta_{k+1} & =  \mathop{\arg\min}_{\theta} \mathop{\mathbb{E}}_{\substack{s\sim d^{\pi} \\a\sim \pi_{k}}}[-r(\theta)A^{\pi_k}_R (s,a)] + \kappa\sum_i \max\{0,\\
&  \quad \mathop{\mathbb{E}}_{\substack{s\sim d^{\pi} \\a\sim \pi_{k}}} \big  [r(\theta)A_{C_i}^{\pi_ k} (s,a) \big ] + (1-\gamma)(J_{C_i}(\pi_{k})-d_i)\}.
\end{aligned}
\label{cppo3}
$}
\end{equation}

Intuitively, the penalized term takes effect when the agent violates the corresponding constraint; otherwise, the objective equals the standard policy optimization when all constraints are fulfilled. Below, we theoretically analyze the equivalence between problem~\eqref{cppo2} and problem~(\ref{cppo3}), i.e., the ReLU operator constructs an exact penalty function with a finite penalty factor $\kappa$. To our best knowledge, it is the first work to solve the constrained policy optimization in the perspective of the exact penalty method. 

\begin{theorem}\label{exact-penalty}
Suppose $\bar\lambda$ is the corresponding Lagrange multiplier vector for the optimum of problem (\ref{cppo2}). Let the penalty factor $\kappa$ be a sufficiently large constant ($\kappa \geq ||\bar\lambda||_\infty$), problem (\ref{cppo2}) and problem (\ref{cppo3}) share the same optimal solution set.
\end{theorem}
\begin{proof}
See the supplemental material.
\end{proof}

Notably, the finiteness of penalty factor $\kappa$ in  Theorem~\ref{exact-penalty} is critical for policy updates because popular methods like square loss function and log-barrier function require the penalty term to go towards positive infinity for an infeasible solution. Otherwise, the optimal solution is not exactly the same as problem (\ref{cppo2})~\cite{liu2020ipo}. However, this would bring about numerical issues in practice and result in unbounded approximate error which will be discussed later.

By now, we construct an equivalent unconstrained objective with ReLU operators and a finite penalty factor $\kappa$. It is still intractable to solve problem (\ref{cppo3}) directly over the unknown future state distribution $s\sim d^\pi$ except for cumbersome off-policy evaluation~\cite{jiang2016doubly} from current policy $\pi_k$. We replace $s \sim d^\pi$ with $s\sim d^{\pi_k}$ and minimize the unconstrained objective over collected trajectories:
\begin{equation}
\resizebox{.91\linewidth}{!}{$
\begin{aligned}
\theta_{k+1} \!\!& =  \mathop{\arg\min}_{\theta} \mathop{\mathbb{E}}_{\substack{s\sim d^{\pi_k} \\a\sim \pi_{k}}}[-r(\theta)A^{\pi_k}_R (s,a)] + \kappa\sum_i \max\{0,\\
&  \quad \mathop{\mathbb{E}}_{\substack{s\sim d^{\pi_k} \\a\sim \pi_{k}}} \big  [r(\theta)A_{C_i}^{\pi_ k} (s,a) \big ] + (1-\gamma)(J_{C_i}(\pi_{k})-d_i)\}.
%& \mathrm{s.t.}  \quad \mathop{\mathbb{E}}_{s\sim d^{\pi_k}}\big[\mathrm{D}_{KL}(\pi(\theta) || \pi(\theta_k))[s]\big] \leq \delta\\
\end{aligned}
\label{cppo4}
$}
\end{equation}
This results in biases w.r.t problem (\ref{cppo3}), and we further provide the analysis of approximate error for that replacement.
%%%%%%%%%%%%%%%%%%%%%%%%%%%%%%%%%%%%%%%%%%%%%%%%%%%%%%%%%%%%%%%%%%%%%%%%%%%%%%%%%%%%%%%%%%%%%%%%%%%%%%%%%%%%%%%%%%%%%%%%%%%%
%%%%%%%%%%%%%%%%%%%%%%%%%%%%%%%%%%%%%%%%%%%%%%%%%%%%%%%%%%%%%%%%%%%%%%%%%%%%%%%%%%%%%%%%%%%%%%%%%%%%%%%%%%%%%%%%%%%%%%%%%%%%
\begin{algorithm}[tb]
\caption{P3O: Penalized Proximal Policy Optimization} 
\label{algo:cppo}
\textbf{Input}: initial policy $\pi_{\theta_0}$, value function $V^{\phi_0}_R$ and each cost-value function $V^{\psi^i_0}_{C_i}, \forall i.$
\begin{algorithmic}[1] %[1] enables line numbers
\FOR{$ k\mathrm{\  in\ } 0,1,2,...K$ }
        	\STATE Collect set of trajectories $\mathcal{D}_k$ with policy $\pi_{\theta_k}$.
        	\STATE Compute $\hat{R}_t = \sum^{T-t}_{k=0} \gamma^k R_{t+k}$ and $\hat{A}^{\pi_{\theta_k}}_R(s_t,a_t)$.
        	\STATE Compute $\hat{C}^i_t = \sum^{T-t}_{k=0} \gamma^k C^i_{t+k}$ and $\hat{A}^{\pi_{\theta_k}}_{C_i}(s_t,a_t), \forall i$.
            \STATE Update $\pi_{\theta_{k+1}}$ using Algorithm \ref{algo:exact}.
            \STATE $\phi_{k+1} = \mathop{\arg\min}_{\phi} \frac{1}{|\mathcal{D}_k| T}\sum(V^{\phi_k}_R(s_t)-\hat{R}_t)^2$.
            \STATE $\phi'_{k+1} = \mathop{\arg\min}_{\phi'} \frac{1}{|\mathcal{D}_k| T}\sum(V^{\psi^i_k}_{C_i}(s_t)-\hat{C}^i_t)^2, \forall i$.
          \ENDFOR
\end{algorithmic}
\textbf{Output}: Optimal policy $\pi_{\theta_K}$.
\end{algorithm}

%%%%%%%%%%%%%%%%%%%%%%%%%%%%%%%%%%%%%%%%%%%%%%%%%%%%%%%%%%%%%%%%%%%%%%%%%%%%%%%%%%%%%%%%%%%%%%%%%%%%%%%%%%%%%%%%%%%%%%%%%%%
\begin{algorithm}[tb]
\caption{Exact Penalized Policy Search Algorithm} 
\label{algo:exact}
\textbf{Input}: current policy $\pi_{\theta_k}$.$\quad \quad \quad$
\begin{algorithmic}[1] %[1] enables line numbers
\FOR{$ n \mathrm{\  in\ } 0,1,2,...N$ }
            \STATE Compute $\mathcal{L}^{\mathrm{CLIP}}_{R}(\theta)$ in Eq.~(\ref{L_R_CLIP}).
            \STATE Compute $\mathcal{L}^{\mathrm{CLIP}}_{C_i}(\theta),\ \forall i = 1,2...m $ in Eq.~(\ref{L_C_CLIP}).
            \STATE $\theta \leftarrow \theta  - \eta \nabla  \mathcal{L}^{\mathrm{P3O}}(\theta)$ in Eq.~(\ref{cppo5}).
            \STATE $\kappa \leftarrow \min\{\rho \kappa, \kappa_\mathrm{max}\} \quad (\rho > 1)$.\\
            \IF{$\mathop{\mathbb{E}}_{s\sim d^{\pi}}\big[\mathrm{D}_{KL}(\pi_\theta || \pi_{\theta_k})[s]\big] \notin [\delta^-_k,\delta^+_k]$}
                \STATE Break.\\
            \ENDIF
          \ENDFOR
\STATE $\theta_{k+1} \leftarrow \theta$.
\end{algorithmic}
\textbf{Output}: next policy $\pi_{\theta_{k+1}}$.
\end{algorithm}

%%%%%%%%%%%%%%%%%%%%%%%%%%%%%%%%%%%%%%%%%%%%%%%%%%%%%%%%%%%%%%%%%%%%%%%%%%%%%%%%%%%%%%%%%%%%%%%%%%%%%%%%%%%%%%%%%%%%%%%%%%%

%%%%%%%%%%%%%%%%%%%%%%%%%%%%%%%%%%%%%%%%%%%%%%%%%%%%%%%%%%%%%%%%%%%%%%%%%%%%%%%%%%%%%%%%%%%%%%%%%%%%%%%%%%%%%%%%%%%%%%%%%%%%%%%%%%%%%%%%%%%%%%%%%%

\begin{figure*}
  \centering
  \subcaptionbox{Circle\label{fig:circle}}
    {\includegraphics[width=3.75cm,height=2.75cm]{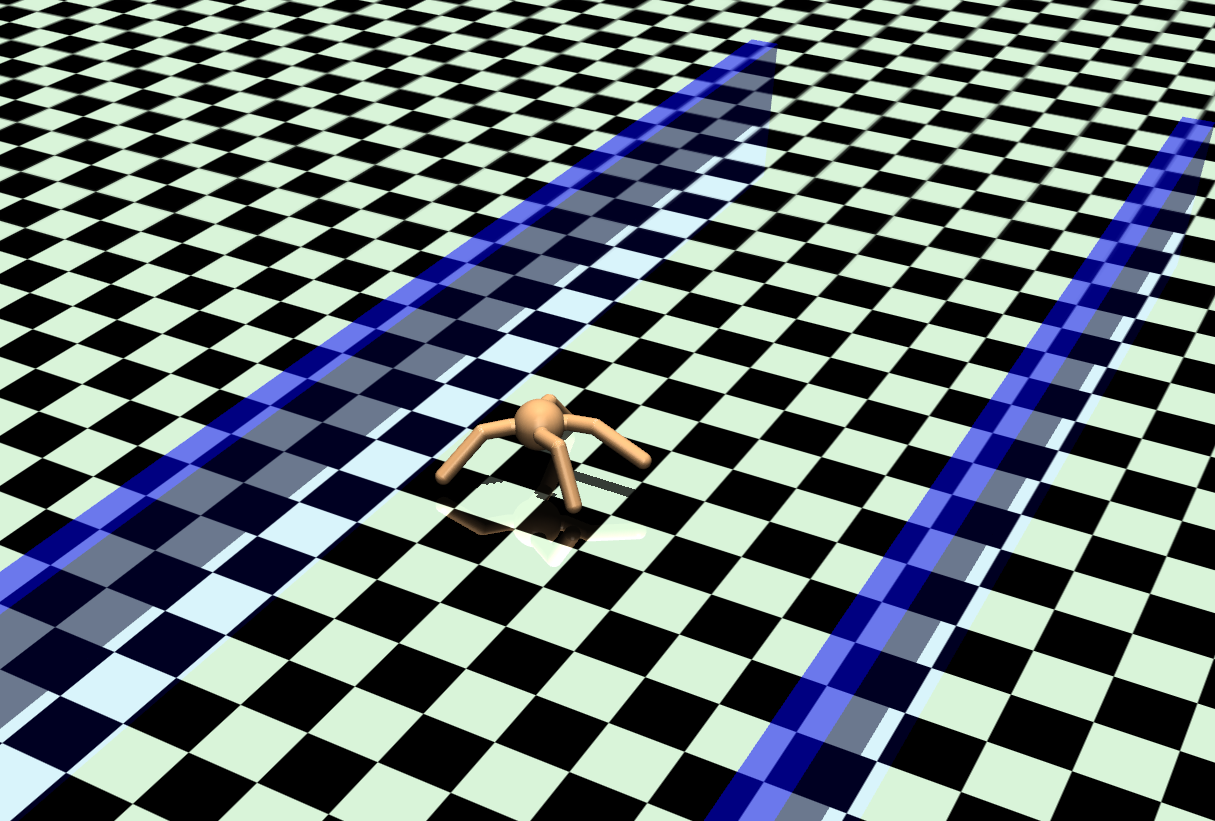}}\hspace{5mm}
  \subcaptionbox{Gather\label{fig:gather}}
    {\includegraphics[width=3.75cm,height=2.75cm]{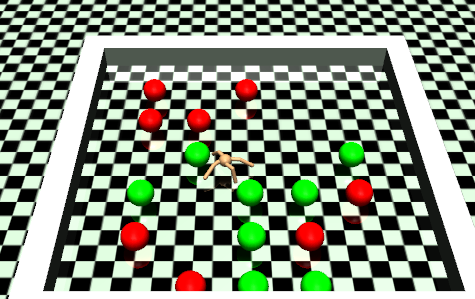}}\hspace{5mm}
  \subcaptionbox{Navigation\label{fig:multi}}
    {\includegraphics[width=3.75cm,height=2.75cm]{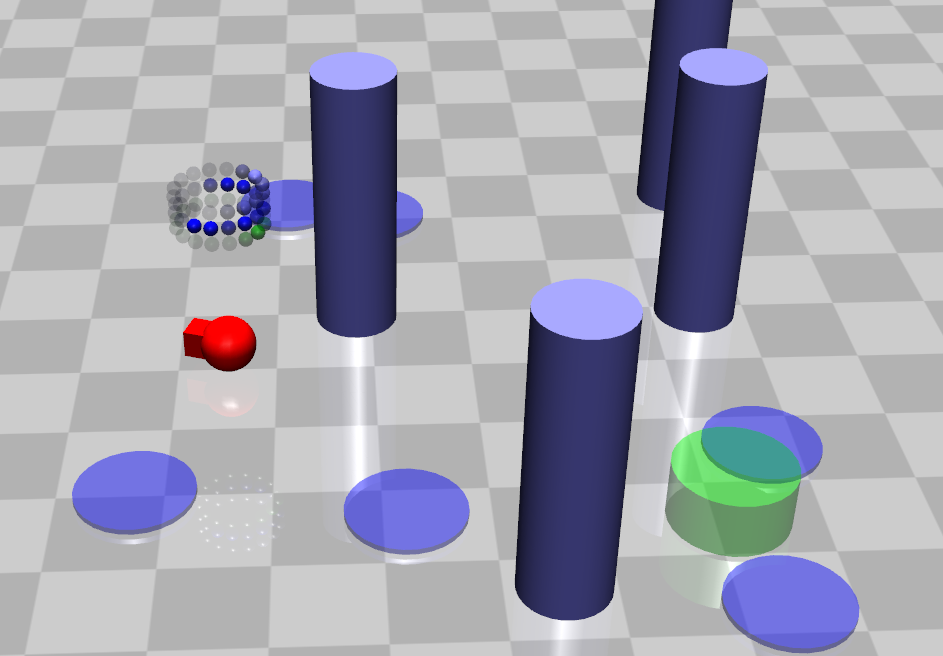}}\hspace{5mm}
\subcaptionbox{Simple Spread\label{fig:multiagent}}
    {\includegraphics[width=3.5cm,height=2.75cm]{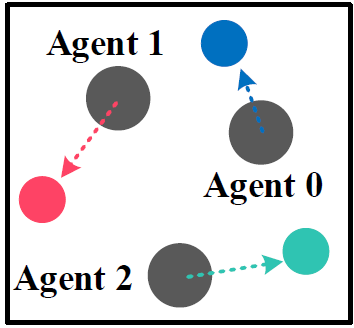}}
  \caption{Experimental benchmarks. (a) Circle: The agent is rewarded for moving in a specified wide circle, but is constrained to stay within a safe region smaller than the radius of the circle. (b) Gather: The agent is rewarded for gathering green apples, but is constrained to avoid red bombs. (c) Navigation: The agent is rewarded for reaching the target area(green) but is constrained to avoid virtual hazards(light purple) and impassible pillars(dark purple). Note that the cost for hazards and pillars are calculated separately and hold different upper limits. (d) Simple Spread: Agents are rewarded for reaching corresponding destinations, but are constrained to the mutual collision. The observation of each agent is not shared in the CMDP execution.}
  \label{fig:exp_benchmarks}
\end{figure*}
%%%%%%%%%%%%%%%%%%%%%%%%%%%%%%%%%%%%%%%%%%%%%%%%%%%%%%%%%%%%%%%%%%%%%%%%%%%%%%%%%%%%%%%%%%%%%%%%%%%%%%%%%%%%%%%%%%%%%%%%%%%%%%%%%%%%%%%%%%%%%%%%%%

\begin{theorem}\label{proximal-error}
For any policy $\pi$ and the given $\pi_k$, define \[\delta = \mathop{\mathbb{E}}_{s\sim d^{\pi_k}}\big[\mathrm{D}_{KL}(\pi || \pi_k)[s]\big].\]
The worst-case approximate error is $\mathcal{O}(\kappa m\sqrt{\delta})$  if we replace problem (\ref{cppo3}) with problem (\ref{cppo4}).
\end{theorem}
\begin{proof}
See the supplemental material.
\end{proof}

Theorem \ref{proximal-error} indicates the necessity of a finite penalty factor $\kappa$, or the approximate error would not be controlled.  On the other hand, we have to restrict policy updates within the neighborhood of last $\pi_k$ (i.e., $\delta \rightarrow 0$) and improve the policy iteratively. Here, we incorporate the trust-region constraint via clipped surrogate objective \cite{schulman2017proximal} in the approximate exactly penalized problem \eqref{cppo4} to guarantee a proximal policy iteration. 
Then, the final optimization objective, abbreviated as P3O (penalized PPO), is derived as:
\begin{equation}
\mathcal{L}^\mathrm{P3O}(\theta) =   \mathcal{L}^{\mathrm{CLIP}}_R(\theta) + \kappa\cdot \sum_i \max\{0,\mathcal{L}^{\mathrm{CLIP}}_{C_i}(\theta) \},
\label{cppo5}
\end{equation}
 where
\begin{equation}
\begin{aligned}
\mathcal{L}^{\mathrm{CLIP}}_{R}(\theta) = & \mathop{\mathbb{E}}_{\substack{s\sim d^{\pi_k} \\a\sim \pi_{k}}}\big [ -\min\big \{  r(\theta)  A_R^{\pi_k} (s,a),\\ 
& \mathrm{clip} (r(\theta),  1-\epsilon,1+\epsilon) A_R^{\pi_k} (s,a) \big \} \big ] \ ;
\end{aligned}
\label{L_R_CLIP}
\end{equation}
\begin{equation}
\begin{aligned}
& \mathcal{L}^{\mathrm{CLIP}}_{C_i}(\theta) = \mathop{\mathbb{E}}_{\substack{s\sim d^{\pi_k} \\a\sim \pi_{k}}}\big  [ \max\big \{ r(\theta) A_{C_i}^{\pi_k} (s,a),
\mathrm{clip} (r(\theta), \\ & 1-\epsilon,1+\epsilon)A_{C_i}^{\pi_k} (s,a) \big \}+ (1-\gamma)(J_{C_i}(\pi_{k})-d_i)\big].
\end{aligned}
\label{L_C_CLIP}
\end{equation}

Based on the clipped surrogate objective technique, we solve the difficult constrained problem (\ref{cppo2}) by alternately estimating and minimizing the unconstrained loss function (\ref{cppo5}).

The remaining problem is how to obtain the finite penalty factor $\kappa$ due to the variety of the magnitude of $\mathcal{L}^\mathrm{CLIP}_{R}(\theta) $ and each $\mathcal{L}^{\mathrm{CLIP}}_{C_i}(\theta)$ that depends on different tasks and up-to-date policies. 
The value of $\kappa$ is required to be larger than the unknown greatest Lagrange multiplier according to Theorem~\ref{exact-penalty}, but Theorem~\ref{proximal-error} implies too large $\kappa$ decays the performance. Thus, there is a trade-off among the hyper-parameter tuning. As shown in Algorithm \ref{algo:exact}, we increase $\kappa$ at every time step, and the early stopping condition is fulfilled when the distance between solutions of two adjacent steps is small enough or the current policy is out of the trust region. In practice, we utilize the normalization trick that maps the advantage estimation to an approximate standard normal distribution regardless of the tasks themselves. We find this technique enables a fixed $\kappa$ for general good results across different tasks. Experimental results show that both of above algorithms work effectively and the learning processes are stable in a wide range of $\kappa$.

Note that the first term in the P3O objective (\ref{cppo5}) is the same as standard PPO and the rest terms are the clipped constraint objectives that prevent the new policy from jumping out of the trust region during the penalty reduction. 
The loss function in P3O is differentiable almost everywhere and can be solved easily via the simple first-order optimizer~\cite{kingma2014adam}. Conclusively,  P3O inherits the benefits of PPO in CMDPs such as better accuracy and sample efficiency than the approximating quadratic optimization~\cite{achiam2017constrained,yang2020projection}. Additionally, the P3O algorithm is naturally scalable to multi-constraint scenarios by adding penalized terms onto the objective (\ref{cppo5}). The pseudo-code of the P3O algorithm for solving general multi-constraint CMDPs is shown in Algorithm \ref{algo:cppo}.
In the end, we extend the P3O algorithm to the decentralized multi-agent setting in which the CMDP is partially observable and the observation can not be shared between each agent. Due to the space limitation, detailed algorithms are provided in the supplementary material.

\section{Experiments}\label{sec:5}
%%%%%%%%%%%%%%%%%%%%%%%%%%%%%%%%%%%%%%%%%%%%%%%%%%%%%%%%%%%%%
\begin{figure*}
  \centering
  \includegraphics[width=0.85\linewidth]{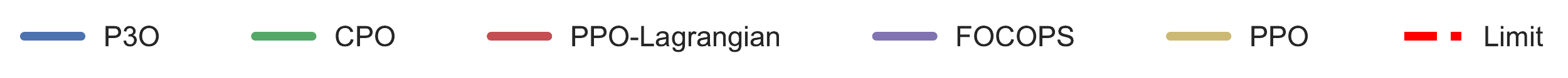}
   \subcaptionbox{Episode Return(AntCircle)\label{fig:circle-epret}}
    {\includegraphics[width=0.24\linewidth,trim=10 15 0 0,clip]{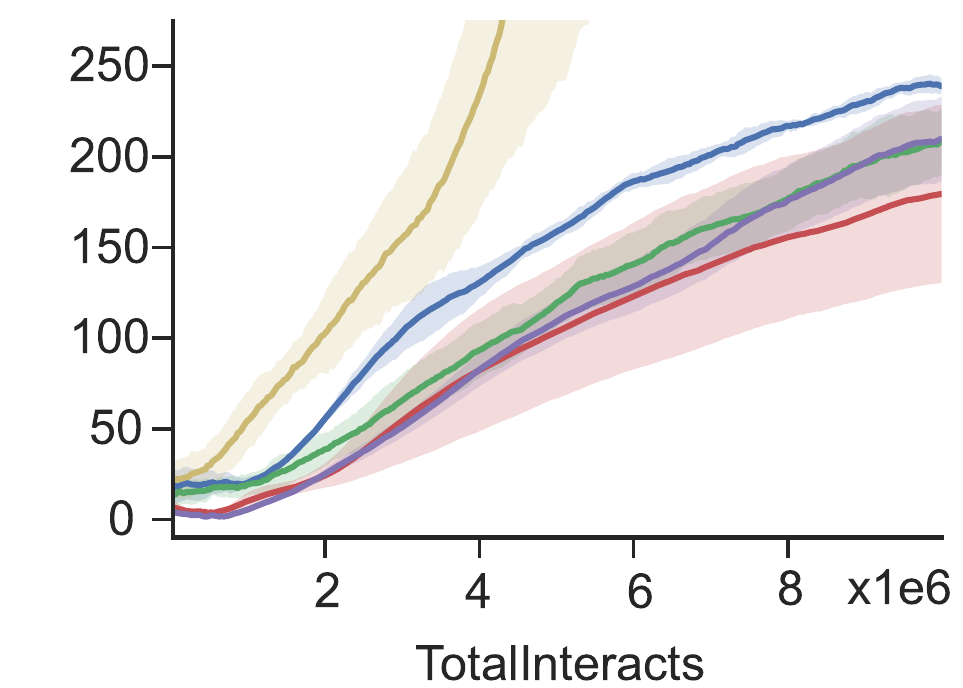}}
      \subcaptionbox{Episode Cost(AntCircle)\label{fig:circle-epcost}}
    {\includegraphics[width=0.24\linewidth,trim=10 15 0 0,clip]{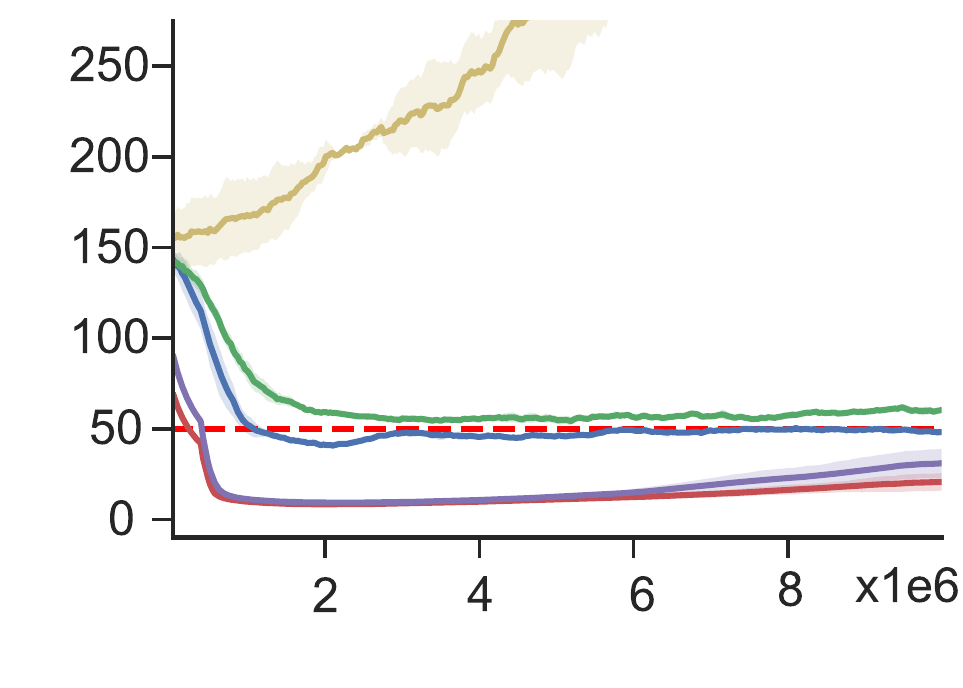}}
    \subcaptionbox{Episode Return(PointGather)\label{fig:gather-epret}}
    {\includegraphics[width=0.24\linewidth,trim=10 15 0 0,clip]{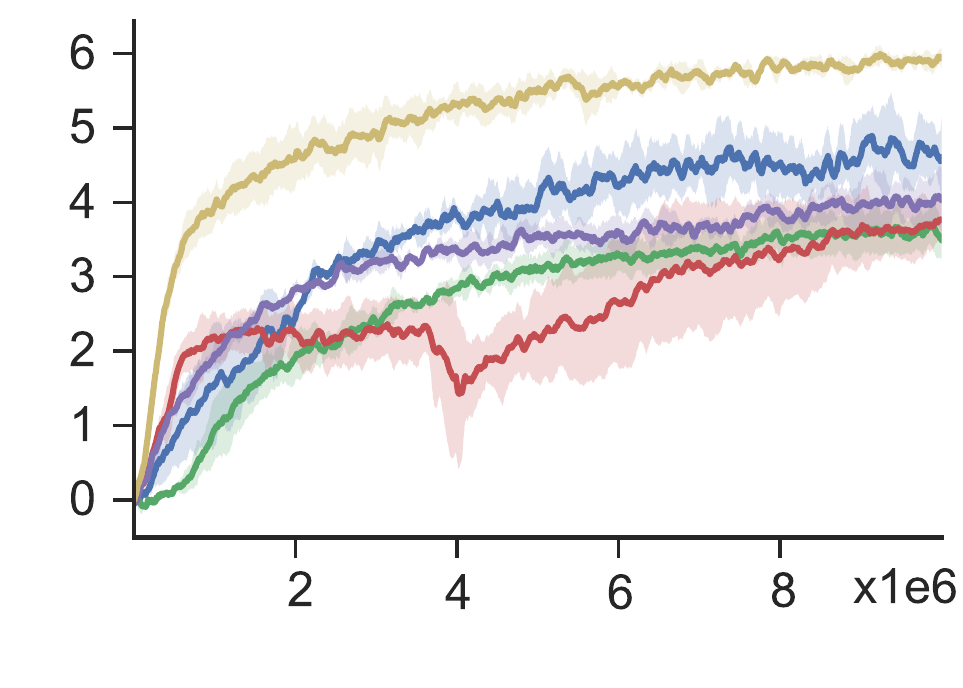}}
      \subcaptionbox{Episode Cost(PointGather)\label{fig:gather-epcost}}
    {\includegraphics[width=0.24\linewidth,trim=10 15 0 0,clip]{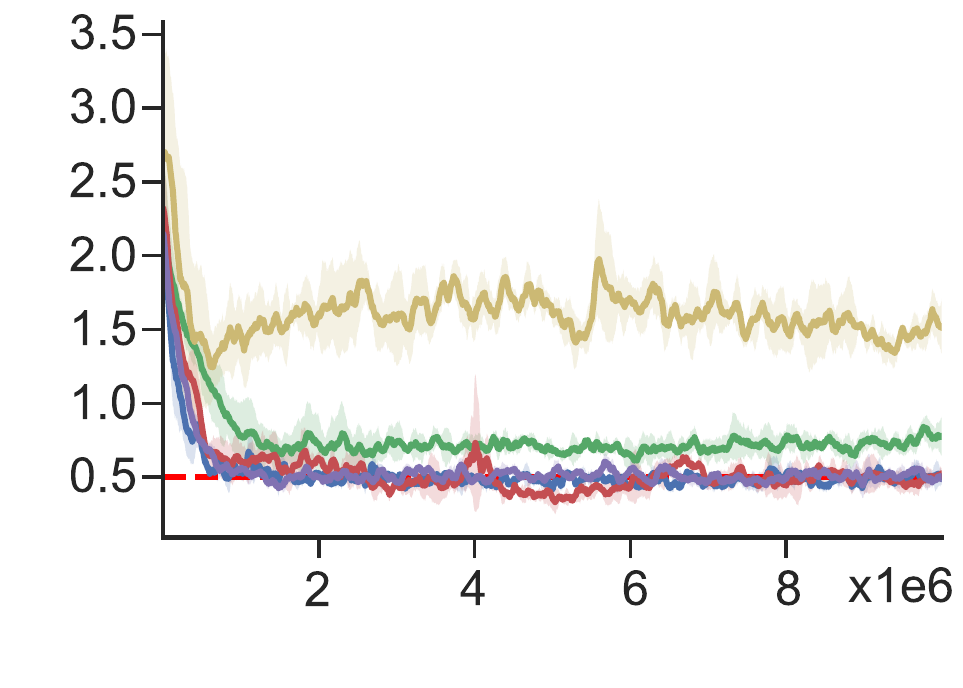}}

  \caption{Average episode return and cost in the single-constraint scenario. The x-axis is the number of interactions with the emulator. The y-axis is the average reward/cost-return. The solid line is the mean and the shaded area is the standard deviation. Each tested algorithm runs over five different seeds. The dashed line in the cost plot is the constraint threshold which is 50 for Circle and 0.5 for Gather.}
  \label{fig:single_result}
\end{figure*}
%%%%%%%%%%%%%%%%%%%%%%%%%%%%%%%%%%%%%%%%%%%%%%%%%%%%%%%%%%%
In this section, we empirically demonstrate the efficacy of the P3O algorithm from the following four perspectives:
\begin{itemize}
    \item P3O outperforms the state-of-the-art algorithms, e.g. CPO~\cite{achiam2017constrained}, PPO-Lagrangian~\cite{ray2019benchmarking} and FOCOPS~\cite{zhang2020first} in safe RL benchmarks with a single constraint.

    \item P3O is robust to more stochastic and complex environments~\cite{ray2019benchmarking} where previous methods fail to simultaneously improve the return effectively and satisfy the constraint strictly.
    
    \item P3O introduces only one additional hyper-parameter (i.e., the penalty factor $\kappa$), which is easy to tune and insensitive to different settings and constraint thresholds.
    
    \item P3O can be extended to multi-constraint and multi-agent scenarios that are barely tested in the previous work.
\end{itemize}

\begin{table*}
\centering
\small
\begin{tabular}{cccccc}
\hline
\!\!\!Task \!\!\!&  & P3O (Ours) &\!\!\! CPO~\cite{achiam2017constrained} & \!\!\! PPO-L~\cite{ray2019benchmarking} & \!\!\! FOCOPS~\cite{zhang2020first}\!\!\\
\hline
\multirow{2}*{\!\!\!AntCircle} & Reward & $253.92\pm 11.67$& $214.45\pm 26.32$& $189.82\pm52.12$ & $232.08\pm 19.87$\\
~ &\!\!\! Cost($\le50$) & $46.18\pm7.49$ & $57.66\pm9.42$ &$27.71\pm7.33$ & $34.50\pm11.21$\\
\hline
\multirow{2}*{\!\!\!PointGather} & Reward & $4.57\pm0.54$& $3.11\pm0.29$ &$3.49\pm0.33$ & $3.14\pm0.79$ \\
~ &\!\!\! Cost($\le 0.5$) & $0.51\pm0.10$& $0.73\pm0.17$ & $0.48\pm0.09$& $0.51\pm0.14$\\
\hline
\end{tabular}
\caption{Mean performance with normal 95\% confidence interval on single-constraint tasks.}
\label{tab:perf}
\end{table*}
%%%%%%%%%%%%%%%%%%%%%%%%%%%%%%%%%%%%%%%%%%%%%%%%%%%%%%%%%
\begin{figure}
      \centering
      ~~~~~\includegraphics[width=0.95\linewidth]{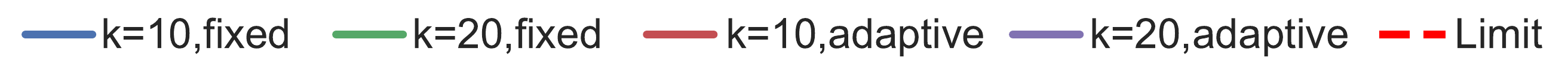}
      \subcaptionbox{Episode Return\label{fig:sa-epret}}
        {\includegraphics[width=0.45\linewidth,trim=10 15 0 0,clip]{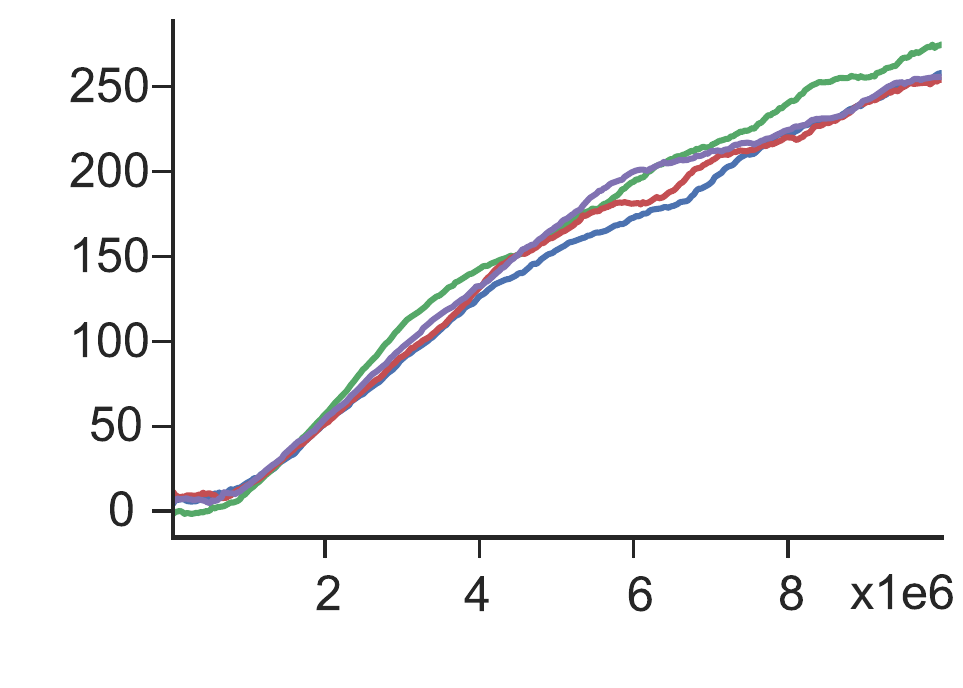}}
      \subcaptionbox{Episode Return\label{fig:sa-epcost}}
        {\includegraphics[width=0.45\linewidth,trim=10 15 0 0,clip]{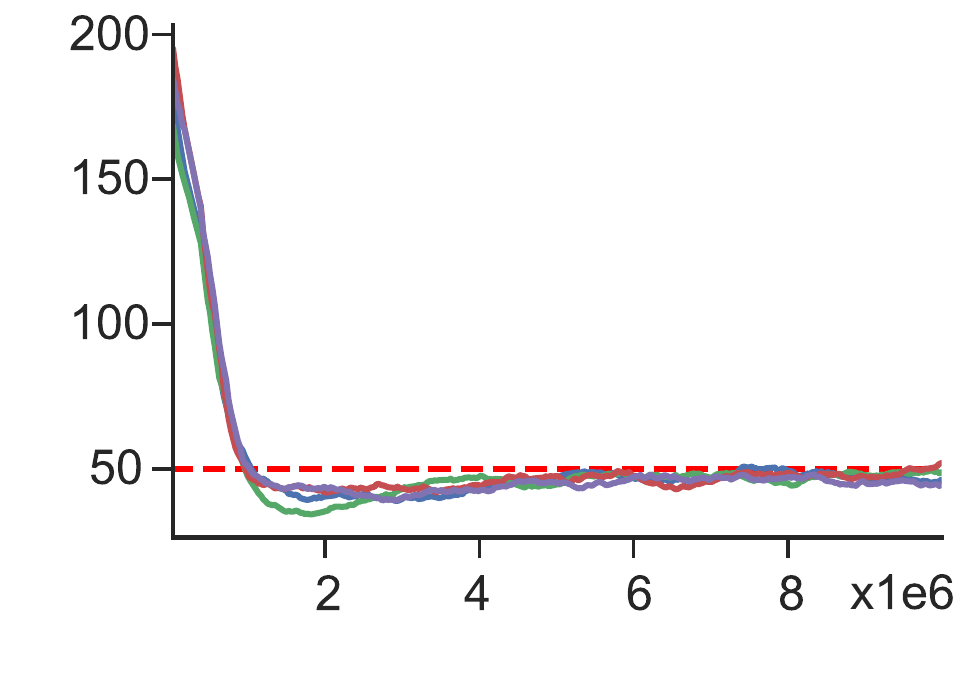}}
      \caption{Performance of P3O for different $\kappa$ settings on AntCircle.}
      \label{fig:sa_result}
\end{figure}
%%%%%%%%%%%%%%%%%%%%%%%%%%%%%%%%%%%%%%%%%%%%%%%%%%%%%%%%%

%\paragraph{Experiment Details.}
We design and conduct experiments in 2 single-constraint (\textit{Circle} and \textit{Gather}), 1 multi-constraint (\textit{Navigation}) and 1 multi-agent (\textit{Simple Spread}) safe RL environments respectively, as illustrated in Figure \ref{fig:exp_benchmarks}.

To be fair in comparison, the proposed P3O algorithm and FOCOPS~\cite{zhang2020first} are implemented with same rules and tricks on the code-base of \citeauthor{ray2019benchmarking}\shortcite{ray2019benchmarking}\footnote{https://github.com/openai/safety-starter-agents} for benchmarking safe RL algorithms. Also, we take standard PPO as the reference which ignores any constraint and serves as an upper bound baseline on the reward performance. 

More information about experiment environments and detailed parameters are provided in the supplementary material.

\paragraph{Single-Constraint Scenario.} 

% Learning curves of average episode return and cost for single-constraint scenarios are shown in Figure \ref{fig:single_result}. 
The numerical results of different algorithms are listed in Table \ref{tab:perf}.
The proposed P3O algorithm has +15\%, +24\%, +9\% higher reward improvement over CPO, PPO-Lagrangian, FOCOPS on the AntCircle task, and +34\%, +23\%, +31\% on the PointGather task respectively. Meanwhile, P3O converges to the upper limit of the safety constraint more tightly which enlarges the parametric search space, as well as satisfies the hard constraint.

As illustrated in Figure \ref{fig:single_result}, P3O is the best for improving the policy while satisfying the given constraint. Conversely, CPO has more constraint violations and even fails to satisfy the constraint after convergence possibly due to Taylor's approximation. This is especially prominent in complex environments like SafetyGym~\cite{ray2019benchmarking}  shown in Figure~\ref{fig:safetygym_result}. The non-negligible approximate error prevents the CPO agent from fully ensuring constraints on virtually over these environments whereas it enjoys a meaningless high return. 
As for PPO-Lagrangian, it is too conservative for the trade-off between reward improvement and constraint satisfaction. Besides, it fluctuates mostly on the learning curves which means it is less steady on the learning process. FOCOPS is better-performed than the two baselines above. However, the hyper-parameter $\lambda$ takes a great effect on the final performance. At last, both PPO-Lagrangian and FOCOPS are more or less sensitive to the initial value as well as the learning rate of the Lagrange multiplier. Improper hyper-parameters would cause catastrophic performance on reward and constraint value.

%%%%%%%%%%%%%%%%%%%%%%%%%%%%%%%%%%%%%%%%%%%%%%%%%%%%%%%%%
\begin{figure}
      \centering
      \includegraphics[width=0.8\linewidth]{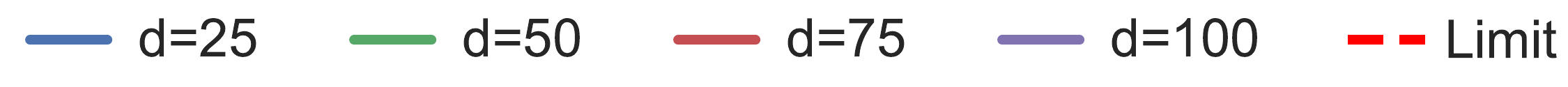}
      \subcaptionbox{Episode Return\label{fig:sac-epret}}
        {\includegraphics[width=0.45\linewidth,trim=10 15 0 0,clip]{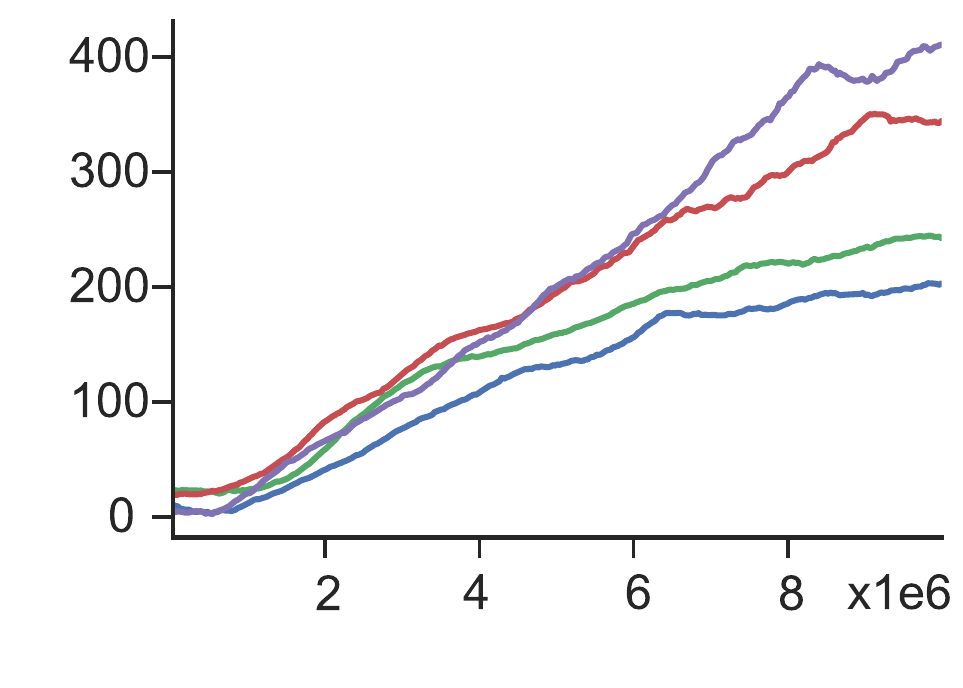}}
      \subcaptionbox{Episode Return\label{fig:sac-epcost}}
        {\includegraphics[width=0.45\linewidth,trim=10 15 0 0,clip]{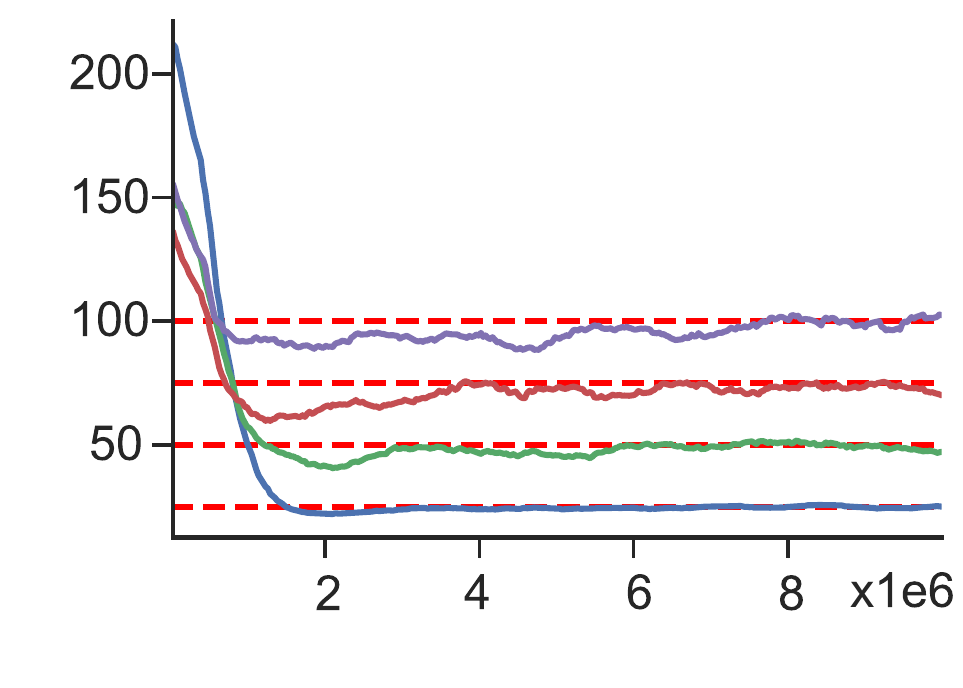}}
      \caption{Performance of P3O for different cost limit $d$ on AntCircle.}
      \label{fig:sac_result}
\end{figure}
%%%%%%%%%%%%%%%%%%%%%%%%%%%%%%%%%%%%%%%%%%%%%%%%%%%%%%%%%

\begin{figure*}
  \centering
  \includegraphics[width=0.76\linewidth]{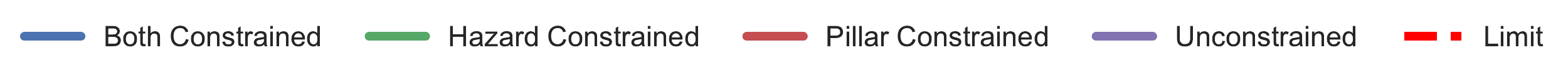}
  \subcaptionbox{Episode Return\label{fig:multi-epret}}
    {\includegraphics[width=0.24\linewidth,trim=10 15 0 0,clip]{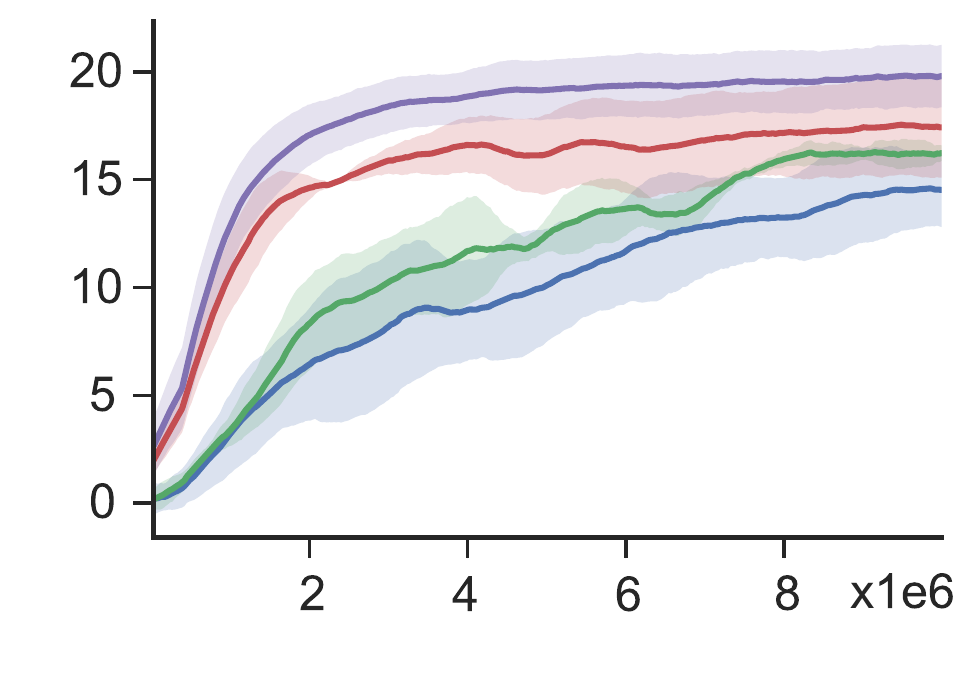}}\hspace{5mm}
  \subcaptionbox{Episode Cost for hazards\label{fig:multi-epcost1}}
    {\includegraphics[width=0.24\linewidth,trim=10 15 0 0,clip]{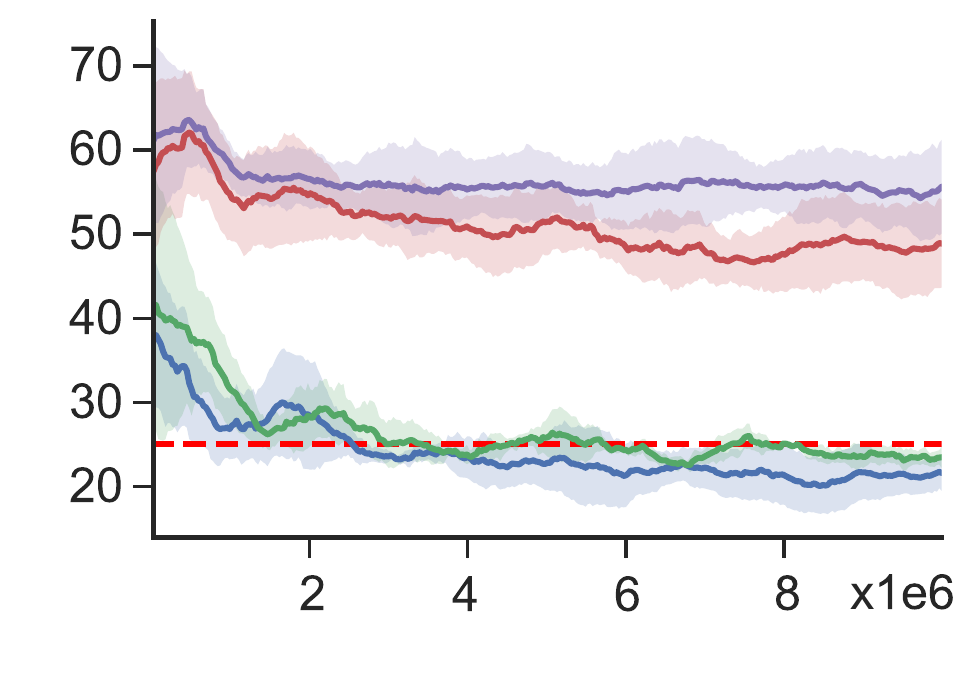}}\hspace{5mm}
  \subcaptionbox{Episode Cost for pillars\label{fig:multi-epcost2}}
    {\includegraphics[width=0.24\linewidth,trim=10 15 0 0,clip]{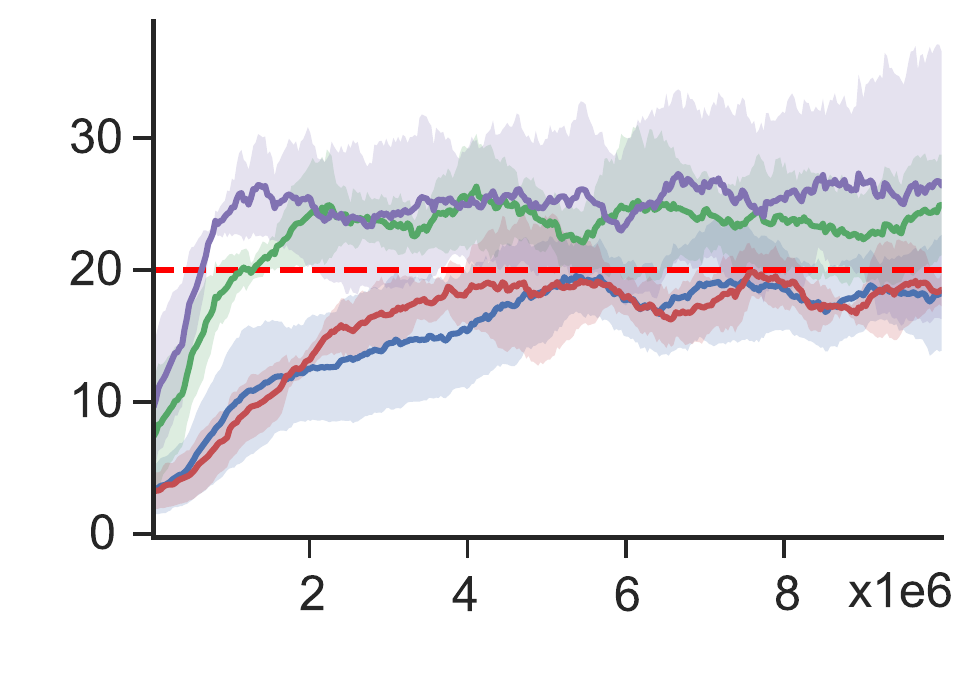}}
  \caption{Average episode return(left), cost1(center, for hazards) and cost2(right, for pillars) in the multiple-constraint scenario. The dashed line in the cost plot is the constraint threshold which is 25 for cost1 and  20 for cost2. Hazard/Pillar constrained means only taking cost1/cost2 into P3O loss function whereas ignoring the other one. }
  \label{fig:multi_result}
\end{figure*}
\paragraph{Sensitivity Analysis.}
The penalty factor $\kappa$ plays a critical role in P3O, 
 which is supposed to vary across different tasks depending on the specific value of $A_R$ and each $A_{c_i}$. As shown in Algorithm ~\ref{algo:exact}, we can increase $\kappa$ at every time steps. Otherwise we can map the advantage estimation to an approximate standard normal distribution by applying advantage normalizing tricks and use a fixed $\kappa$ for general good results. The experimental results are stable in a wide range of $\kappa$, as illustrated in Figure~\ref{fig:sa_result}.

We further verify that P3O works effectively for different threshold levels by changing the cost limit $d$. Figure~\ref{fig:sac_result} shows P3O can effectively learn constraint-satisfying policies for all the cases. Notably, P3O precisely converges to the cost upper limit. By contrast, the baseline algorithms are too conservative for policy updates and have more overshoot in the process of cost violation reduction.

\begin{figure}
      \centering
      ~~~~~~\includegraphics[width=0.9\linewidth]{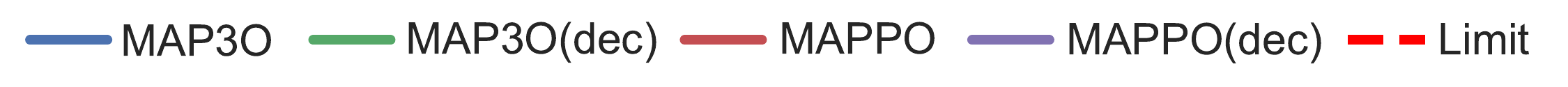}
      \subcaptionbox{Episode Return\label{fig:mpe-epret}}
        {\includegraphics[width=0.45\linewidth,trim=10 15 0 0,clip]{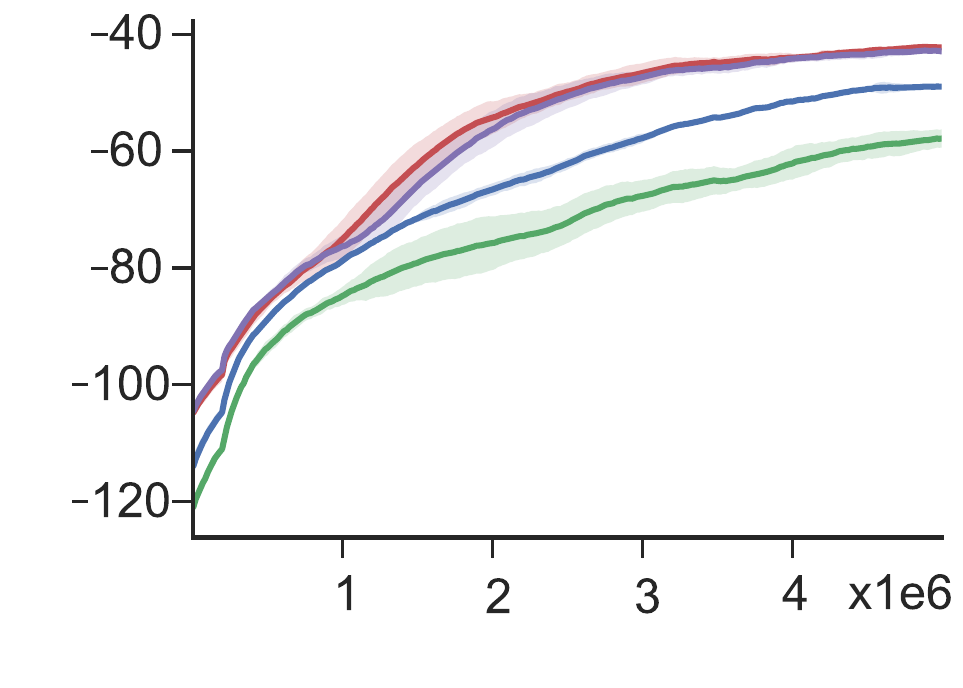}}
      \subcaptionbox{Episode Return\label{fig:mpe-epcost}}
        {\includegraphics[width=0.45\linewidth,trim=10 15 0 0,clip]{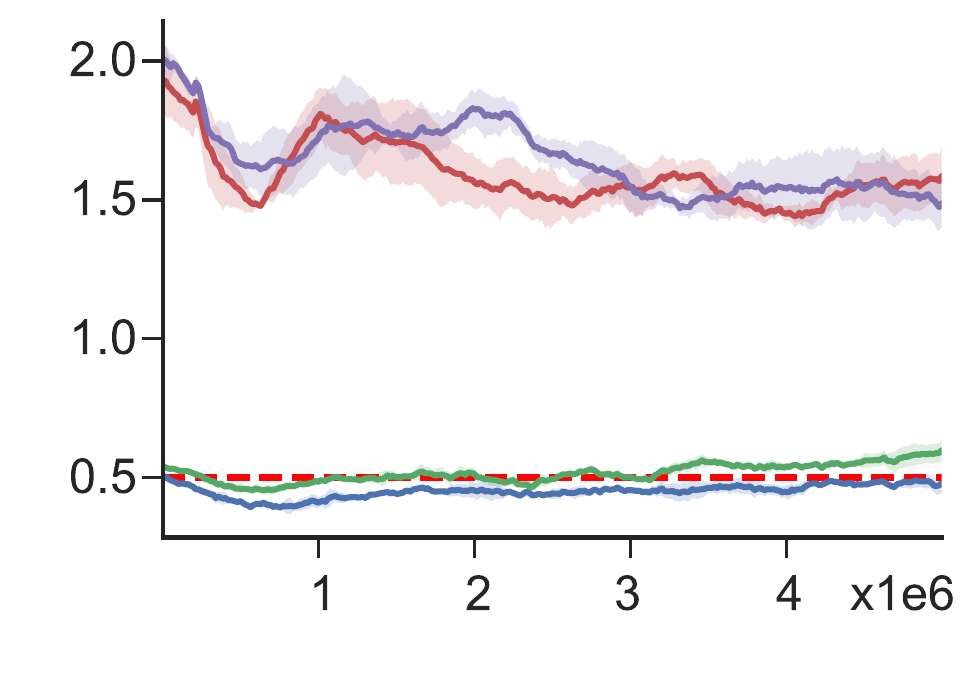}}
      \caption{Average episode return and cost in the multi-agent scenario. The algorithm is decentralized with notation (dec).}
      \label{fig:mpe_result}
\end{figure}

\paragraph{Multi-Constraint Scenario.}

It is convenient to extend P3O to multi-constraint scenarios by simply adding exact penalty functions with the minor modification on the loss function. 

Learning curves of average episode return and two independent costs for the navigation task are shown in Figure \ref{fig:multi_result}. The unconstrained algorithm is the original PPO. The hazard (pillar) constrained algorithm only takes the cost of hazard (pillar) into P3O loss function whereas ignoring the other one. It can be seen that those single-constrained algorithms maintain the specified constraint under its upper limit but fail to control the other one. When both constraints are incorporated into our algorithm, each one is well satisfied at the convergence. The episode return is lower than unconstrained and single-constrained settings, which is reasonable because the safety constraints are more strict.

It is remarkable that at the beginning of training, the agent violates the hazard constraint but barely hits pillars because there are more small hazards around the starting point and the pillars are nearly unreachable by random moving due to the initialization mechanism. However, with the increasing velocity, the agent is less likely to satisfy the pillar constraint. Satisfaction on the two types of constraints shows that P3O admits both feasible and infeasible initial policies and finds a constraint-satisfying solution in the end.

\paragraph{Multi-Agent Scenario.}
We also study Multi-Agent Penalized Policy Optimization    (MAP3O) in the collaborative task and take unconstrained Multi-Agent PPO (MAPPO)~\cite{chao2021surprising} as the reference. To tackle the issue of partial observability, recurrent networks are used in both algorithms.

The given task is a decentralized partially observable Markov decision process (DEC-POMDP) in which each agent has an exclusive observation and a shared global reward/cost. We replace the MAPPO loss function with Eq.~(\ref{cppo5}) and conduct a fully decentralized training and
execution. The experimental results show that MAP3O can improve the total reward while ensuring the mutual collision constraint. Inspired by \citeauthor{chao2021surprising}\shortcite{chao2021surprising}, we enhance MAP3O by learning a centralized critic for both global reward and cost that takes in the concatenation of each agent's local observation. The centralized version of MAP3O has better reward performance and satisfies the constraint more strictly. The learning curves of average episode return and cost are shown in Figure \ref{fig:mpe_result}.

\begin{figure}
      \centering
      \includegraphics[width=0.8\linewidth]{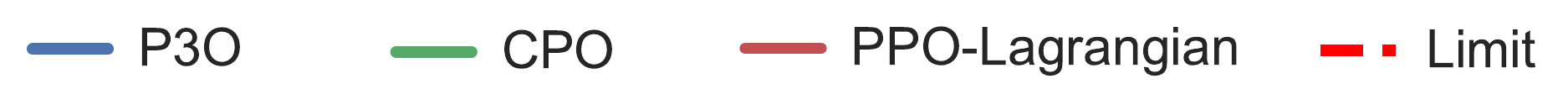}
      %\vspace{-0.12cm}
      \subcaptionbox{Episode Return\label{fig:sg-epret}}
        {\includegraphics[width=0.45\linewidth,trim=10 15 0 0,clip]{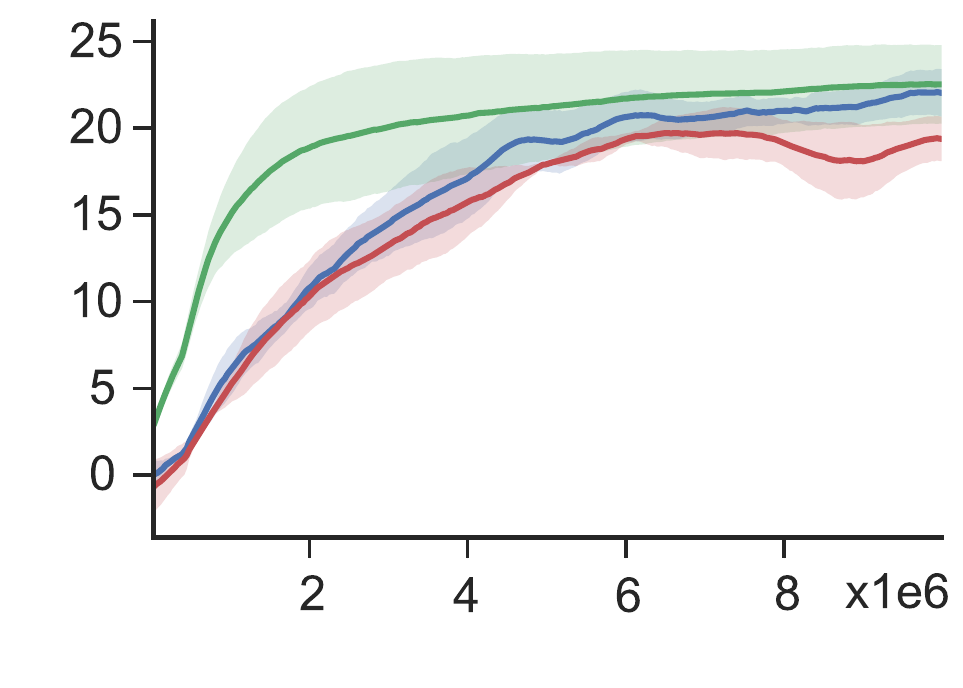}}
      \subcaptionbox{Episode Return\label{fig:sg-epcost}}
        {\includegraphics[width=0.45\linewidth,trim=10 15 0 0,clip]{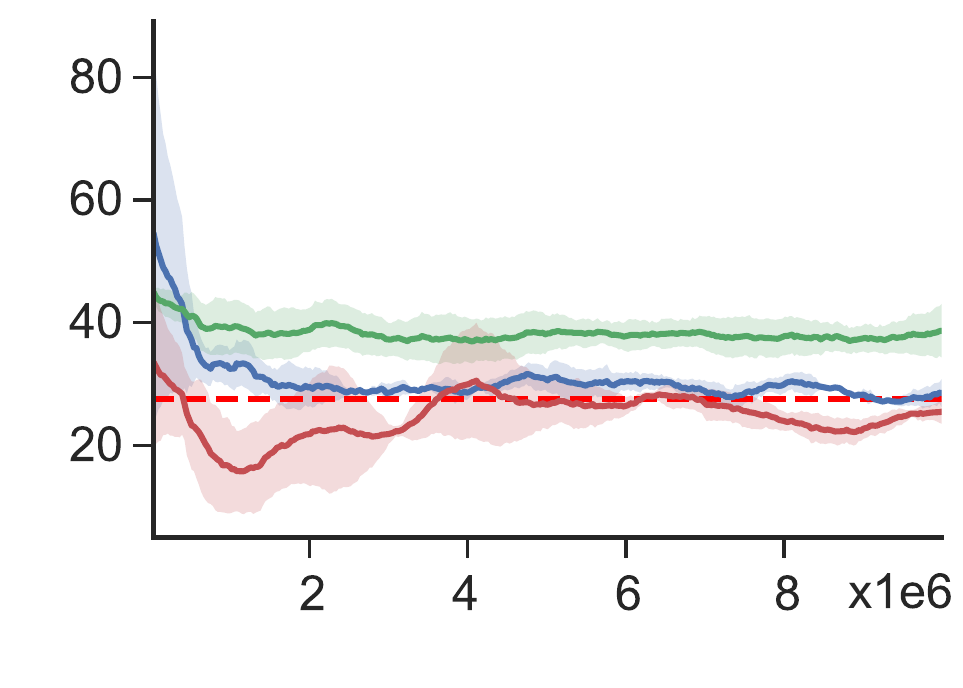}}
      \caption{Average performance for P3O, CPO and PPO-Lagrangian on Safetygym PointGoal.}
      \label{fig:safetygym_result}
\end{figure}

\section{Conclusion}
In this paper, we propose Penalized Proximal Policy Optimization (P3O), an effective and easy-to-implement algorithm for safe reinforcement learning. P3O has an unconstrained optimization objective due to the exact penalty reformulation and guarantees better performance on reward and constraint value empirically. Meanwhile, it is naturally compatible with multi-constraint and multi-agent scenarios. P3O adopts first-order minimization and avoids the inversion of high-dimensional Hessian Matrix, which is especially beneficial to large CMDPs with deep neural networks. Thus we consider the future work as end-to-end visual tasks which are less studied in safe reinforcement learning.

\section*{Acknowledgements}
This work is supported by  Science and Technology Innovation 2030 –“Brain Science and Brain-like Research” Major Project (No. 2021ZD0201402 and No. 2021ZD0201405).

\bibliographystyle{named}
\bibliography{ijcai22}

\setcounter{section}{0}
\setcounter{subsection}{0}
\section*{Supplementary Material}
To make this work self-contained and easy to read, the supplementary material is organized as follows:  In the first 3 sections, we provide the detailed proof of proposition and theorems claimed in the main paper. In Section 4, we propose the Multi-agent P3O algorithm in detail. In Section 5, we give the implementation details of experiment environments, agents, and algorithms.

\section{Proof of Proposition 1.}
The following lemma adapts Performance Difference Lemma~\cite{kakade2002approximately} to more general settings.

\begin{lemma}
For any function $f:  S \times  A \times  S  \mapsto \mathbb{R}$ and any policies $\pi$ and $\pi'$, define
$J_f(\pi) = \mathop{\mathbb{E}}_{\tau\sim \pi}\big [ \Sigma^\infty_{t=0}\gamma^tf(s_t,a_t,s_{t+1})\big ]$ and $d^\pi(s)  =  (1-\gamma) \sum^\infty_{t=0} \gamma^t P(s_t=s | \pi)$. Then the following equation holds:
\[
J_f(\pi')-J_f(\pi)=\frac{1}{1-\gamma}\mathop{\mathbb{E}}_{\substack{s\sim d^{\pi'}\\  a\sim \pi'}} [ A_f^{\pi} (s,a)]
\]
\label{performance_lemma}
\end{lemma}
\begin{proof}
First, $J_f(\pi')$ can be reformulated as:
\begin{equation*}
\resizebox{.97\linewidth}{!}{$
\begin{aligned}
& J_f(\pi') = \mathop{\mathbb{E}}_{\tau\sim \pi'}\big [ \sum^\infty_{t=0}\gamma^tf(s_t,a_t,s_{t+1})\big ]\\ &=\mathop{\mathbb{E}}_{\tau\sim \pi'}\big [ \sum^\infty_{t=0}\gamma^tf(s_t,a_t,s_{t+1}) -\sum^\infty_{t=0}\gamma^tV_f^\pi(s_t)  +\sum^\infty_{t=0}\gamma^tV_f^\pi(s_t) \big ]\\ &=\mathop{\mathbb{E}}_{\tau\sim \pi'}\big [ \sum^\infty_{t=0}\gamma^t\big (f(s_t,a_t,s_{t+1}) -V_f^\pi(s_t) \\
& \qquad\qquad\qquad\qquad\qquad\qquad+\gamma V_f^\pi(s_{t+1}) \big )\big ]+\mathop{\mathbb{E}}_{\tau\sim \pi'}\big [V_f^\pi(s_0)\big ]\\ & =\mathop{\mathbb{E}}_{\tau\sim \pi'}\big [ \sum^\infty_{t=0}\gamma^t A_f^\pi(s_t,a_t)\big ]+\mathop{\mathbb{E}}_{s_0\sim \mu}\big [V_f^\pi(s_0)\big ]\\ & = \mathop{\mathbb{E}}_{\tau\sim \pi'}\big [ \sum^\infty_{t=0}\gamma^t A_f^\pi(s_t,a_t)\big ]+ J_f(\pi)
\end{aligned}
$}
\end{equation*}
Next, replace $\tau\sim\pi'$ with $s\sim d^{\pi'}$ and $a\sim\pi'$:
\begin{equation*}
\resizebox{.97\linewidth}{!}{$
\begin{aligned}
& \quad \mathop{\mathbb{E}}_{\tau\sim \pi'}\big [ \sum^\infty_{t=0}\gamma^t A_f^\pi(s_t,a_t)\big ]\\ &=\mathbb{E}\big [ \sum^\infty_{t=0} \sum_s P(s_t=s)\sum_a \pi(a_t=a|s_t=s)\gamma^t A_f^\pi(s_t,a_t)\big ] \\ &=\mathbb{E}\big [ \big ( \sum_s \sum^\infty_{t=0}\gamma^tP(s_t=s)\big)\sum_a \pi(a_t=a|s_t=s)A_f^\pi(s_t,a_t)\big ] \\ &=\frac{1}{1-\gamma}\mathop{\mathbb{E}}_{\substack{s\sim d^{\pi'} \\a\sim \pi'}}\big [  A_f^\pi(s_t,a_t)\big ]
\end{aligned}
$}
\end{equation*}
Thus, the following equation holds: \[J_f(\pi')-J_f(\pi)=\frac{1}{1-\gamma}\mathop{\mathbb{E}}_{\substack{s\sim d^{\pi'}\\  a\sim \pi'}} [ A_f^{\pi} (s,a)].\] 
The proof of Lemma~\ref{performance_lemma} is completed. 
\end{proof}

For a better readability, we rewrite Proposition 1 as below.
\begin{aproposition}
The new policy $\pi_{k+1}$  obtained from the current policy $\pi_k$ via the following problem:
\begin{equation}
%\resizebox{.91\linewidth}{!}{$
\begin{aligned}
&\pi_{k+1}  =  \mathop{\arg\max}_{\pi}\mathop{\mathbb{E}}_{\substack{s\sim d^{\pi} \\a\sim \pi}}[ A^{\pi_k}_R (s,a)]\\
& \mathrm{s.t.}  \quad J_{C_i}(\pi_{k})+ \frac{1}{1-\gamma}\mathop{\mathbb{E}}_{\substack{s\sim d^{\pi} \\a\sim \pi}} \big  [A_{C_i}^{\pi_ k} (s,a) \big ] \leq d_i,\ \forall i.\\
\end{aligned}
%$}
\label{appendixP1}
\end{equation}
yields a monotonic return improvement and hard constraint satisfaction.

\label{th:appendix1}
\end{aproposition}
\begin{proof}[Proof of Proposition 1]
The goal is to obtain new policy $\pi_{k+1}$ which is the solution of the following problem:
\begin{equation}
%\resizebox{.91\linewidth}{!}{$
\begin{aligned}
&\pi_{k+1}  =  \mathop{\arg\max}_{\pi} J_R(\pi)\\
& \mathrm{s.t.}  \quad J_{C_i}(\pi) \leq d_i,\ \forall i.\\
\end{aligned}
%$}
\end{equation}
Replace $f$ with reward function $R$ and each cost function $C_i$ respectively in Lemma~\ref{performance_lemma}, we have:
\begin{equation}
J_R(\pi)-J_R(\pi_k)=\frac{1}{1-\gamma}\mathop{\mathbb{E}}_{\substack{s\sim d^{\pi}\\  a\sim \pi}} [ A_R^{\pi_k} (s,a)]
\end{equation}
\begin{equation}
J_{C_i}(\pi)-J_{C_i}(\pi_k)=\frac{1}{1-\gamma}\mathop{\mathbb{E}}_{\substack{s\sim d^{\pi'}\\  a\sim \pi'}} [ A_{C_i}^{\pi_k} (s,a)], \ \forall i
\end{equation}
Thus, the above problem (\ref{appendixP1}) can be reformulated as:
\begin{equation}
%\resizebox{.91\linewidth}{!}{$
\begin{aligned}
&\pi_{k+1}  =  \mathop{\arg\max}_{\pi}\big( J_R(\pi_k) + \frac{1}{1-\gamma}\mathop{\mathbb{E}}_{\substack{s\sim d^{\pi} \\a\sim \pi}}[ A^{\pi_k}_R (s,a)] \big)\\
& \mathrm{s.t.}  \quad J_{C_i}(\pi_{k})+ \frac{1}{1-\gamma}\mathop{\mathbb{E}}_{\substack{s\sim d^{\pi} \\a\sim \pi}} \big  [A_{C_i}^{\pi_ k} (s,a) \big ] \leq d_i,\ \forall i.\\
\end{aligned}
%$}
\end{equation}
The proof is completed by omitting the constant term in the optimization objective.
\end{proof}

%%%%%%%%%%%%%%%%%%%%%%%%%%%%%%%%%%%%%%%%%%%%%%%%%%%%%%%%%%%%%%%%%%%%%%%%%%%%%%%%%%%%%%%%%%%%%%%%%%%%%%%%%%%%%%%%

\section{Proof of Theorem 2.}
We rewrite problem (\ref{cppo2}) as:
\begin{equation}
\resizebox{.385\linewidth}{!}{$
\begin{aligned}
 &\theta_{k+1} = \mathop{\arg\min}_{\theta} \mathcal{L}_R(\theta)\\
& \mathrm{s.t.} \quad \mathcal{L}_{C_i}(\theta) \leq 0, \ \forall i\\
\end{aligned}
$}
 \tag{P}
 \label{P}
\end{equation}
and rewrite problem (\ref{cppo3}) as:
\begin{equation}
\theta_{k+1} = \mathop{\arg\min}_{\theta}  \mathcal{L}_R(\theta) + \kappa\sum_i \max\{0,\mathcal{L}_{C_i}(\theta)\}
\tag{Q}
\label{Q}
\end{equation}
Here we shorthand $\mathcal{L}_R(\theta) =  \mathop{\mathbb{E}}_{\substack{s\sim d^{\pi} \\a\sim \pi_{k}}}[-r(\theta)A^{\pi_k}_R (s,a)]$ and 
$\mathcal{L}_{C_i}=\mathop{\mathbb{E}}_{\substack{s\sim d^{\pi} \\a\sim \pi_{k}}} \big  [r(\theta)A_{C_i}^{\pi_ k} (s,a)\big ] + (1-\gamma)(J_{C_i}(\pi_{k})-d_i), \ \forall i$.
\begin{lemma}
Suppose $\bar{\theta}$ is the optimum of the constrained problem (\ref{P}). Let $\bar\lambda$ be the corresponding Lagrange multiplier vector for its dual problem. Then for $\kappa \geq ||\bar\lambda||_\infty$, $\bar{\theta}$ is also a minimizer of its penalized optimization problem (\ref{Q}).
\label{exact_penalty_function_lemma1}
\end{lemma}
\begin{proof}
Since $\kappa \geq ||\bar\lambda||_\infty$, by taking $[\cdot]^+$ as shorthand for $\max\{0,\cdot\}, $it follows that:
\begin{align}\label{th1:1}
    \mathcal{L}_R(\theta) + \kappa\sum_i \mathcal{L}^+_{C_i}(\theta) 
    &\geq \mathcal{L}_R(\theta) + \sum_i \bar{\lambda_i}\mathcal{L}^+_{C_i}(\theta)\\ 
    &\geq \mathcal{L}_R(\theta) + \sum_i \bar{\lambda_i}\mathcal{L}_{C_i}(\theta)\label{th1:2}
\end{align}

By assumption, $\bar\theta$ is a Karush-Kuhn-Tucker point in the constrained problem (\ref{P}), at which KKT conditions are satisfied with the Lagrange multiplier vector  $\bar{\lambda}$. It gives:
\begin{align}\label{th1:3}
    \mathcal{L}_R(\theta) + \sum_i \bar{\lambda_i}\mathcal{L}_{C_i}(\theta)
    &\geq \mathcal{L}_R(\bar\theta) + \sum_i \bar{\lambda_i}\mathcal{L}_{C_i}(\bar\theta)\\ \label{th1:4}
    &=  \mathcal{L}_R(\bar\theta) + \sum_i \bar{\lambda_i}\mathcal{L}^+_{C_i}(\bar\theta)\\\label{th1:5}
     &= \mathcal{L}_R(\bar\theta) + \kappa \sum_i \mathcal{L}^+_{C_i}(\bar\theta)
\end{align}
(\ref{th1:3}) holds because $\bar\theta$ minimizes the Lagrange function. (\ref{th1:4}) is derived from the complementary slackness.

Then, by (\ref{th1:1})-(\ref{th1:5}), we conclude that $\mathcal{L}(\theta) \geq \mathcal{L}(\bar\theta)$ holds for all $\theta \in \Theta$, which means $\bar\theta$ is a minimizer of  the penalized optimization problem (\ref{Q}). 

The proof of Lemma~\ref{exact_penalty_function_lemma1} is completed.
\end{proof}

\begin{lemma}
Let $\widetilde{\theta}$ be a minimizer of the penalized optimization problem (\ref{Q}). $\bar\theta$ and $\bar\lambda$ are the same as they are defined in Theorem~\ref{exact_penalty_function_lemma1}. Then for $\kappa \geq ||\bar\lambda||_\infty$, $\widetilde{\theta}$ is also optimal in the constrained problem (\ref{P}).
\label{exact_penalty_function_lemma2}
\end{lemma}
\begin{proof}
If $\widetilde\theta$ is in the set of feasible solutions $D=\{\theta | \mathcal{L}_{C_i}(\theta) \leq 0, \ \forall i\}$, it gives:
\begin{align}
\mathcal{L}_R(\widetilde\theta) &= \mathcal{L}_R(\widetilde\theta) + \kappa\sum_i \mathcal{L}^+_{C_i}(\widetilde\theta)\\
&\leq \mathcal{L}_R(\theta) + \kappa\sum_i \mathcal{L}^+_{C_i}(\theta) =\mathcal{L}_R(\theta)
\end{align}
The inequality above indicates  $\widetilde{\theta}$ is also optimal in the constrained problem (\ref{P}). 

If  $\widetilde\theta$ is not feasible, we have:
\begin{align}
 &\mathcal{L}_R(\bar\theta) + \kappa \sum_i \mathcal{L}^+_{C_i}(\bar\theta)
    =  \mathcal{L}_R(\bar\theta) + \sum_i \bar{\lambda_i}\mathcal{L}^+_{C_i}(\bar\theta) \\
    & = \mathcal{L}_R(\bar\theta) + \sum_i \bar{\lambda_i}\mathcal{L}_{C_i}(\bar\theta)
    \leq \mathcal{L}_R(\widetilde\theta) + \sum_i \bar{\lambda_i}\mathcal{L}_{C_i}(\widetilde\theta) \\
   & \leq \mathcal{L}_R(\widetilde\theta) + \sum_i \bar{\lambda_i}\mathcal{L}^+_{C_i}(\widetilde\theta)
   \leq \mathcal{L}_R(\widetilde\theta) + \kappa\sum_i \mathcal{L}^+_{C_i}(\widetilde\theta) 
\end{align}which is a contradiction to the assumption that  $\widetilde{\theta}$ is a minimizer of the penalized optimization problem (\ref{Q}).Thus, $\widetilde\theta$ can only be the feasible optimal solution for problem (\ref{P}).

The proof of Lemma~\ref{exact_penalty_function_lemma2} is completed.
\end{proof}

For a better readability, we rewrite Theorem 2 as below.
\begin{atheorem}
Suppose $\bar\lambda$ is the corresponding Lagrange multiplier vector for the optimum of problem (\ref{P}). Let the penalty factor $\kappa$ be a sufficiently large constant ($\kappa \geq ||\bar\lambda||_\infty$), the two problems (\ref{P}) and  (\ref{Q}) share the same optimal solution set.
\end{atheorem}
\begin{proof}[Proof of Theorem 2]
The proof is completed with Lemma~\ref{exact_penalty_function_lemma1} and Lemma~\ref{exact_penalty_function_lemma2}.
\end{proof}

%%%%%%%%%%%%%%%%%%%%%%%%%%%%%%%%%%%%%%%%%%%%%%%%%%%%%%%%%%%%%%%%%%%%%%%%%%%%%%%%%%%%%%%%%%%%%%%%%%%%%%%%%%%%%%%%
\section{Proof of Theorem 3.}
We rewrite problem (\ref{cppo3}) as:
\begin{equation}
\mathcal{L}^\pi(\theta)  = \mathcal{L}^\pi_R(\theta) +\kappa\sum_i \max\{0,\mathcal{L}^\pi_{C_i}(\theta)\}
\tag{M}
\label{M}
\end{equation}
Here we shorthand $\mathcal{L}^\pi_R(\theta) =  \mathop{\mathbb{E}}_{\substack{s\sim d^{\pi} \\a\sim \pi_{k}}}[-r(\theta)A^{\pi_k}_R (s,a)]$ and 
$\mathcal{L}^\pi_{C_i}(\theta)=\mathop{\mathbb{E}}_{\substack{s\sim d^{\pi} \\a\sim \pi_{k}}} \big  [r(\theta)A_{C_i}^{\pi_ k} (s,a)\big ] + (1-\gamma)(J_{C_i}(\pi_{k})-d_i), \ \forall i$.

We rewrite problem (\ref{cppo4}) as:
\begin{equation}
\mathcal{L}^\pi_k(\theta)  = \mathcal{L}^{\pi_k}_R(\theta) +\kappa\sum_i \max\{0,\mathcal{L}^{\pi-K}_{C_i}(\theta)\}
\tag{N}
\label{N}
\end{equation}
Here we shorthand $\mathcal{L}^{\pi_k}_R(\theta) =  \mathop{\mathbb{E}}_{\substack{s\sim d^{\pi_k} \\a\sim \pi_{k}}}[-r(\theta)A^{\pi_k}_R (s,a)]$ and 
$\mathcal{L}^{\pi_k}_{C_i}(\theta)=\mathop{\mathbb{E}}_{\substack{s\sim d^{\pi_k} \\a\sim \pi_{k}}} \big  [r(\theta)A_{C_i}^{\pi_ k} (s,a)\big ] + (1-\gamma)(J_{C_i}(\pi_{k})-d_i), \ \forall i$.

\begin{lemma}For any function $f:  S \times  A \times  S  \mapsto \mathbb{R}$ and any policies $\pi$ and $\pi'$, define
\begin{equation*}
\begin{aligned}
\delta& = \mathop{\mathbb{E}}_{s\sim d^{\pi}}\big[\mathrm{D}_{KL}(\pi' || \pi)[s]\big]\\
\varepsilon^{\pi'}_f &= \mathop{\max}_{s} |\mathop{\mathbb{E}}_{a\sim \pi'} A_f^\pi(s,a)| \\
D^{\pm}_{\pi,f}(\pi') &= \frac{1}{1-\gamma}\mathop{\mathbb{E}}_{\substack{s\sim d^{\pi}\\  a\sim \pi}} \bigg[\frac{\pi'(s,a)}{\pi(s,a)} A^{\pi}_f (s,a) \pm \frac{\gamma\varepsilon^{\pi'}_f\sqrt{2\delta}}{1-\gamma} \bigg]
\end{aligned}
\end{equation*}
The following bounds hold:
\[D^{+}_{\pi,f}(\pi') \geq J_f(\pi') - J_f(\pi) \geq D^{-}_{\pi,f}(\pi')\]
\label{approxiamte_error1}
\end{lemma}
\begin{proof}
See \cite{achiam2017constrained} for details.
\end{proof}

For a better readability, we rewrite Theorem 3 as below.
\begin{atheorem}
For any policies $\pi$, define \[\delta = \mathop{\mathbb{E}}_{s\sim d^{\pi_k}}\big[\mathrm{D}_{KL}(\pi || \pi_k)[s]\big].\]
The worst-case approximate error is $\mathcal{O}(\kappa m\sqrt{\delta})$  if we replace problem (\ref{M}) with problem (\ref{N}).
\label{th:appendix3}
\end{atheorem}
\begin{proof}[Proof of Theorem 3]
According to Lemma~\ref{performance_lemma} and Lemma~\ref{approxiamte_error1}, we have:
\small
\begin{align}
&\quad\ \ |\mathcal{L}^\pi(\theta) - \mathcal{L}^\pi_k(\theta)|\\
&\leq |\mathcal{L}^\pi_R(\theta)- \mathcal{L}^{\pi_k}_R(\theta)|  +\kappa\sum_i |[\mathcal{L}^\pi_{C_i}(\theta)]^+ - [\mathcal{L}^{\pi_k}_{C_i}(\theta)]^+|\\
&\leq |\mathcal{L}^\pi_R(\theta)- \mathcal{L}^{\pi_k}_R(\theta)|  +\kappa\sum_i |\mathcal{L}^\pi_{C_i}(\theta) - \mathcal{L}^{\pi_k}_{C_i}(\theta)|\\
&\leq \frac{\gamma\varepsilon^{\pi'}_R\sqrt{2\delta}}{1-\gamma} +\kappa\sum_i \frac{\gamma\varepsilon^{\pi'}_{C_i}\sqrt{2\delta}}{1-\gamma} \\
&\leq  (1+\kappa m) \frac{\gamma||\varepsilon^{\pi'}||_{\infty}\sqrt{2\delta}}{1-\gamma}.
\end{align}
Then, the proof is completed.
\end{proof}

\section{Multi-Agent P3O}
In this section, we extend P3O to the decentralized multi-agent refinement learning (MARL) scenarios and denote it as MAP3O, inspired by MAPPO~\cite{chao2021surprising}.

A decentralized partially observed constrained Markov decision process (DEC-POCMDP) is defined by a 7-tuple $(S,\{A_i\},T,\{\Omega_i\},O,R,C)$,
where $S$ is the state space, $A_i$ is the action space for agent i, $A = \times_iA_i$ is the joint action space, $\Omega_i$ is the observation space for agent i, $\Omega = \times_i\Omega_i$ is the joint observation space, $T(s,a,s') = P(s'|s,a)$ and $O(s,a,o) = P(o|s',a)$ are transition probabilities, $R$ is the global reward function and $C$ is a set of global constraint functions.
$J_R(\pi) = \mathop{\mathbb{E}}_{\tau\sim \pi}\big [ \sum^\infty_{t=0}\gamma^t R(s_t,\boldsymbol{a_t},s_{t+1})\big ]$
and $J_{C}(\pi) = \mathop{\mathbb{E}}_{\tau\sim \pi}\big [ \Sigma^\infty_{t=0}\gamma^t C(s_t,\boldsymbol{a_t},s_{t+1})\big ]$.
The goal of Multi-agent safe RL is to optimize the joint policy, i.e., $\pi^* = \mathop{\arg\max}_{\pi \in \Pi_C} J_R(\pi)$ where $\Pi_C = \{ \pi \in \Pi \ | J_{C}(\pi) \leq d\}$.

The detailed iterates scheme of fully decentralized MAP3O for agent $i$ is summarized in Algorithm \ref{MAP30-alg1}.

\begin{algorithm}[H]
\caption{MAP3O (dec): Decentralized Multi-Agent Penalized Proximal Policy Optimization} 
\label{MAP30-alg1}
\small
\textbf{Input}: initial policy $\pi(\theta^i_0)$, value function $V_R(\phi^i_0)$ and cost-value function $V_C(\psi^i_0)$.
\begin{algorithmic}[1] %[1] enables line numbers
\FOR{$ k\mathrm{\  in\ } 0,1,2,...K$ }
            \STATE Empty data buffer $\mathcal{D}_k$.
            \WHILE{$\mathcal{D}_k$ is not full}
            \STATE Reset a new trajectory $\tau$.
            \STATE Initialize actor RNN state $h^i_{0,\pi}$.
            \STATE Initialize critic RNN states $h^i_{0,V_R}$ and $h^i_{0,V_C}$.
            \FOR{$ t\mathrm{\  in\ } 1,2,...T$ }
            \STATE $p^i_t, h^i_{t,\pi} = \pi(o^i_t,h^i_{t-1,\pi}; \theta^i_{t-1})$.
            \STATE Sample $a^i_t$ over $p^i_t$.
            \STATE $v^i_{t,R},h^i_{t,V_R} = V_R(o^i_t,h^i_{t-1,V_R}; \phi^i_{t-1})$.
            \STATE $v^i_{t,C},h^i_{t,V_C} = V_C(o^i_t,h^i_{t-1,V_C}; \psi^i_{t-1})$.
            
            \STATE Execute $a^i_t$ and record $r_t, c_t, o^i_{t+1}$.
            \STATE $\tau += [o^i_t,a^i_t,h^i_{t,\pi},v^i_{t,R},h^i_{t,V_R},v^i_{t,C},h^i_{t,V_C},r_t,c_t,o^i_{t+1}]$.
            \ENDFOR
        	\STATE Compute every $\hat{R}_t = \sum^{T-t}_{k=0} \gamma^k r_{t+k}$ and $\hat{A}^i_R(o^i_t,a^i_t)$.
        	\STATE Compute every  $\hat{C}_t = \sum^{T-t}_{k=0} \gamma^k c_{t+k}$ and $\hat{A}^i_C(o^i_t,a^i_t)$.
        	\STATE Add data into $\mathcal{D}_k$.
        	\ENDWHILE
        	\FOR{$ n \mathrm{\  in\ } 0,1,2,...N$ }
        	\STATE Update the RNN hidden states in data buffer $\mathcal{D}_k$.
            \STATE Compute $\mathcal{L}^{\mathrm{CLIP}}_{R}(\theta^i_k)$ in Eq.~(\ref{L_R_CLIP}).
            \STATE Compute $\mathcal{L}^{\mathrm{CLIP}}_{C}(\theta^i_k)$ in Eq.~(\ref{L_C_CLIP}).
            \STATE $\theta^i_k \leftarrow \theta^i_k  - \eta \nabla  \mathcal{L}^{\mathrm{P3O}}(\theta^i_k)$ in Eq.~(\ref{cppo5}).
            \ENDFOR
            \STATE $\theta^i_{k+1} \leftarrow \theta^i_{k}$.
            \STATE Update $\phi^i_{k+1}$ and $\psi^i_{k+1}$ with value regression objective.
          \ENDFOR
\end{algorithmic}
\textbf{Output}: Optimal policy $\pi(\theta^i_K)$.
\end{algorithm}

Centralized training and decentralized execution (CTDE) is a popular framework to learn a single policy and avoid the joint action as inputs at the same time. The key technique is to learn a centralized critic that takes in the concatenation of each agent’s local observation, as shown in Algorithm \ref{MAP30-alg2}.
\begin{algorithm}[H]
\caption{MAP3O: Multi-Agent Penalized Proximal Policy Optimization} 
\label{MAP30-alg2}
\small
\textbf{Input}: initial policy $\pi(\theta_0)$, value function $V_R(\phi_0)$ and cost-value function $V_C(\psi_0)$.
\begin{algorithmic}[1] %[1] enables line numbers
\FOR{$ k\mathrm{\  in\ } 0,1,2,...K$ }
            \STATE Empty data buffer $\mathcal{D}_k$.
            \WHILE{$\mathcal{D}_k$ is not full}
            \STATE Reset a new trajectory $\tau$.
            \STATE Initialize each actor RNN state $h^i_{0,\pi}$.
            \STATE Initialize each critic RNN states $h^i_{0,V_R}$ and $h^i_{0,V_C}$.
            \FOR{$ t\mathrm{\  in\ } 1,2,...T$ }
            \FOR{each agent $i$}
            \STATE $p^i_t, h^i_{t,\pi} = \pi(o^i_t,h^i_{t-1,\pi}; \theta_{t-1})$.
            \STATE Sample $a^i_t$ over $p^i_t$.
            \STATE $v^i_{t,R},h^i_{t,V_R} = V_R(\boldsymbol{o}_t,h^i_{t-1,V_R}; \phi_{t-1})$.
            \STATE $v^i_{t,C},h^i_{t,V_C} = V_C(\boldsymbol{o}_t,h^i_{t-1,V_C}; \psi_{t-1})$.
            \ENDFOR
            
            \STATE Execute $\boldsymbol{a}_t$ and record $r_t, c_t, \boldsymbol{o}_{t+1}$.
            \STATE $\tau += [\boldsymbol{o}_t,\boldsymbol{a}_t,\boldsymbol{h}_{t,\pi},\boldsymbol{v}_{t,R},\boldsymbol{h}_{t,V_R},\boldsymbol{v}_{t,C},\boldsymbol{h}_{t,V_C},r_t,c_t,\boldsymbol{o}_{t+1}]$.
            \ENDFOR
        	\STATE Compute every $\hat{R}_t = \sum^{T-t}_{k=0} \gamma^k r_{t+k}$ and $\hat{A}_R(\boldsymbol{o}_t,\boldsymbol{a}_t)$.
        	\STATE Compute every  $\hat{C}_t = \sum^{T-t}_{k=0} \gamma^k c_{t+k}$ and $\hat{A}_C(\boldsymbol{o}_t,\boldsymbol{a}_t)$.
        	\STATE Add data into $\mathcal{D}_k$.
        	\ENDWHILE
        	\FOR{$ n \mathrm{\  in\ } 0,1,2,...N$ }
        	\STATE Update the RNN hidden states in data buffer $\mathcal{D}_k$.
            \STATE Compute $\mathcal{L}^{\mathrm{CLIP}}_{R}(\theta_k)$ in Eq.~(\ref{L_R_CLIP}).
            \STATE Compute $\mathcal{L}^{\mathrm{CLIP}}_{C}(\theta_k)$ in Eq.~(\ref{L_C_CLIP}).
            \STATE $\theta_k \leftarrow \theta_k  - \eta \nabla  \mathcal{L}^{\mathrm{P3O}}(\theta_k)$ in Eq.~(\ref{cppo5}).
            \ENDFOR
            \STATE $\theta_{k+1} \leftarrow \theta_{k}$.
            \STATE Update $\phi_{k+1}$ and $\psi_{k+1}$ with value regression objective.
          \ENDFOR
\end{algorithmic}
\textbf{Output}: Optimal policy $\pi(\theta_K)$.
\end{algorithm}

\section{Implementation Details.}

\begin{table*}
\centering
\small
\caption{Algorithm-specific parameters.}
\setlength{\tabcolsep}{6mm}{
\begin{tabular}{lccccc}
\hline
Parameter & P3O & CPO & PPO-L & FOCOPS & PPO\\
\hline
Number of Hidden layers & 2 & 2 & 2 & 2 & 2 \\
Number of Hidden nodes & 255 & 255& 255& 255& 255\\
Activation Function  & $\tanh$ & $\tanh$ & $\tanh$ & $\tanh$ & $\tanh$ \\
% Rollout Length $T$ & 1e3 & 1e3& 1e3& 1e3& 1e3 \\
% Buffer Size $|\mathcal{D}|$ & 3e4 & 3e4& 3e4& 3e4& 3e4 \\
% Total Interactions & 1e7 & 1e7& 1e7& 1e7& 1e7 \\
Discount Factor $\gamma$ & 0.99 & 0.99 & 0.99 & 0.99& 0.99\\
GAE parameter $\lambda^{\mathrm{GAE}}$ & 0.97 & 0.97& 0.97& 0.97& 0.97\\
Actor Learning Rate $\eta_\pi$    & 3e-4  & N/A & 3e-4& 3e-4& 3e-4  \\
Critic Learning Rate $\eta_{V_R}$  & 1e-3  & 1e-3& 1e-3& 1e-3& 1e-3     \\
Critic Learning Rate $\eta_{V_C}$  & 1e-3  & 1e-3& 1e-3& 1e-3& N/A    \\
Penalty Factor $\kappa$   & 20  & N/A & N/A  & N/A & N/A      \\
Clip Ratio $\epsilon$ & 0.2  & N/A & 0.2  & N/A & 0.2 \\
Trust Region $[\delta^-,\delta^+]$ & [0,1e-2] & [0,1e-2]& [0,1e-2]& [0,1e-2]& [0,1e-2]   \\
Damping Coefficient & N/A & 0.1& N/A& N/A& N/A\\
Backtrack Learning Rate& N/A & 0.8& N/A& N/A& N/A\\
Backtrack Iterations& N/A & 10& N/A& N/A& N/A\\
Penalty Initial Value $\nu_0$ & N/A & N/A & 1& 1& N/A\\
Penalty Learning Rate $\eta_\nu$ & N/A & N/A & 0.05& 0.01& N/A\\
Temperature Factor $\lambda$ & N/A & N/A & N/A & 1.5 & N/A\\
\hline
\end{tabular}}
\label{tab:algo}
\end{table*}

\begin{table*}
\centering
\small
\caption{Experiment-specific parameters.}
\setlength{\tabcolsep}{6mm}{
\begin{tabular}{lccccc}
\hline
Parameter & Circle & Gather & Navigation &  Simple Spread & PointGoal\\
\hline
State Space $|S|$       & $\mathbb{R}^{113}$  & $\mathbb{R}^{86}$& $\mathbb{R}^{60}$& $\mathbb{R}^{54}$  &$\mathbb{R}^{60}$    \\
Action Space $|A|$    & $\mathbb{R}^8$     & $\mathbb{R}^2$  & $\mathbb{R}^2$  & 5  & $\mathbb{R}^2$  \\
Cost Limit $d$ & 50 & 0.5 & 25,20 & 0.5  & 25\\
Rollout Length $T$ & 1e3 & 1e2 & 1e3 & 25 & 1e3 \\
Buffer Size $|\mathcal{D}|$ & 3e4 & 3e3 & 3e4 & 1e3 & 3e4  \\
Total Interactions & 1e7 & 1e7& 1e7 & 5e6 & 1e7\\
\hline
\end{tabular}}
\label{tab:exp}
\end{table*}

\subsection{Environments}
\paragraph{Circle}
This environment is inspired by \citeauthor{achiam2017constrained}~\shortcite{achiam2017constrained}. Reward is maximized by moving along a circle of radius $d$:
\begin{equation*}
    R = \frac{v^\mathrm{T}[-y,x]}{1 + \big| \sqrt{x^2+y^2} - d\big|},
\end{equation*}
but the safety region $x_{\mathrm{lim}}$ is smaller than the radius $d$:
\begin{equation*}
    C = \mathbf{1}[x > x_{\mathrm{lim}}].
\end{equation*}

In our setting, $d=10, x_{\mathrm{lim}}=3$ and the constraint threshold is 50 by default, which indicates the agent is only allowed to move in a narrow space for safety requirements.

\paragraph{Gather} 
This environment is inspired by \citeauthor{achiam2017constrained}~\shortcite{achiam2017constrained}. The agent receives a reward of +1 for collecting an apple (green ball) and a reward of -1 for collecting a bomb (red ball). It also receives a cost of +1 when encountering the bomb. 8 apples and bombs spawn respectively on the map at the start of each episode.

In our settings, we set the constraint threshold to be 0.5. Notably, the punishment in reward function, to some extent, keeps the agent from collecting bombs even trained by the unconstrained PPO. But it still has many safety violations and fails to satisfy the constraint threshold.

\paragraph{Navigation}
This environment is inspired by \citeauthor{ray2019benchmarking}~\shortcite{ray2019benchmarking}. Reward is maximized by getting close to the destination:
\begin{equation*}
    R = \mathrm{Dist}(target,s_{t-1}) - \mathrm{Dist}(target,s_{t}),
\end{equation*}
but it yields a cost of +1 when the agent hits the hazard or the pillar. The two different types of cost functions are returned separately and have different thresholds. In out setting, $d_1 = 25$ for the hazard constraint and $d_2 = 20$ for the pillar constraint.  

\paragraph{Simple Spread} 
This environment is inspired by \citeauthor{lowe2017multi}~\shortcite{lowe2017multi}. In our experiment, 3 agents and 3 endpoints spawn respectively on the 2D arena at the start of each episode.
The agents are collaborative and rewarded for moving towards corresponding endpoints:
\begin{equation*}
    R = -\sum_{i=1}^{3} \mathrm{Dist}(endpoint_i,s_{it}),
\end{equation*}
but it yields a cost of +1 if any pair of agents collide. We set $d = 0.5$.

\subsection{Agents}
For single-constraint scenarios, Point agent is a 2D mass point($A \subseteq  \mathbb{R}^2$) and Ant agent is an quadruped robot($A \subseteq  \mathbb{R}^8$). For the multi-constraint scenario which is modified from OpenAI SafetyGym~\cite{ray2019benchmarking}, $S \subseteq \mathbb{R}^{28 + 16\cdot m}$  where $m$ is the number of pseudo-radars(one for each type of obstacles and we set two different types of obstacles in the Navigation task) and $A \subseteq \mathbb{R}^2$ for a mass point or a wheeled car.
For the multi-agent scenario which is inspired by OpenAI Multi-Agent Particle Environment~\cite{lowe2017multi}, $O \subseteq \mathbb{R}^{18}$ each (the state is partially observable) and $|A|$ is 5 for any particle moving up/down/left/right and remaining still.

\subsection{Parameters}
We summarize the algorithm-specific parameters and experiment-specific parameters in Table \ref{tab:algo} and Table \ref{tab:exp} , respectively.

\clearpage

\end{document}